\documentclass{article}

\PassOptionsToPackage{compress, round}{natbib}
\usepackage[main, final]{neurips_2026}

\makeatletter
\renewcommand{\@noticestring}{}
\makeatother

\usepackage[utf8]{inputenc} 
\usepackage[T1]{fontenc}    
\usepackage{hyperref}       
\usepackage{url}            
\usepackage{booktabs}       
\usepackage{amsfonts}       
\usepackage{nicefrac}       
\usepackage{microtype}      
\usepackage{xcolor}         

\usepackage{amsmath}
\usepackage{amssymb}
\usepackage{amsthm}
\usepackage{mathtools}
\usepackage{thmtools}
\usepackage{thm-restate}
\usepackage{physics}

\usepackage{siunitx}

\usepackage{wrapfig}
\usepackage{caption}

\usepackage{lipsum}
\usepackage{enumitem}
\usepackage{eucal}
\usepackage{algorithm}
\usepackage{algpseudocode}
\usepackage{cleveref}

\declaretheorem[name=Remark]{remark}

\renewcommand{\P}{\mathcal{P}}
\newcommand{\Q}{\mathcal{Q}}
\newcommand{\R}{\mathbb{R}}
\newcommand{\KL}[2]{\mathcal{D}\qty[#1\, |\, #2]}
\newcommand{\W}{\mathcal{W}}
\newcommand{\T}{\mathcal{T}}
\newcommand{\C}{\mathcal{C}}

\newcommand{\one}{\mathbf{1}}

\DeclareMathOperator*{\minimize}{minimize}
\DeclareMathOperator*{\argmin}{arg\ min}
\newcommand{\define}{\coloneqq}

\title{Variational Inference via Entropic Transport Descent}

\author{%
  Vincent Pacelli\\
  School of Aerospace Engineering\\
  Georgia Institute of Technology\\
  Atlanta, GA 30332 \\
  \texttt{vincent@pacel.li}
  \And
  Akash Ratheesh\\
  School of Aerospace Engineering\\
  Georgia Institute of Technology\\
  Atlanta, GA 30332 \\
  \texttt{akashratheesh@gatech.edu}
  \And
  Evangelos Theodorou\\
  School of Aerospace Engineering\\
  Georgia Institute of Technology\\
  Atlanta, GA 30332 \\
  \texttt{evangalos.theodorou@gatech.edu}
}

\crefalias{subappendix}{section}

\begin{document}

\maketitle

\begin{abstract}
Particle-based variational inference (ParVI) methods approximate
an intractable target distribution by evolving an ensemble of
interacting samples.  Existing approaches rely predominantly on
kernel-based repulsion (e.g., SVGD), which suffers from
\emph{variance collapse} in high dimensions and \emph{mode
collapse} on multimodal targets---pathologies caused by the
absence of global transport structure.  We introduce
\emph{entropic transport descent} (ETD), a ParVI family that
frames each particle update as an entropy-regularized optimal
transport problem.  Derived from the JKO proximal scheme by
lifting to the space of couplings and relaxing via the KL chain
rule, each ETD iteration reduces to a Sinkhorn computation.  The
resulting transport plan provides global coordination, guiding
each particle to nearby high-density proposals and naturally
preserving multimodal structure.  ETD can operate entirely
score-free, requiring only pointwise evaluations of the
unnormalized target density.  Experiments on variance-collapse
diagnostics, Bayesian logistic regression, neural networks, and
molecular Boltzmann distributions show that ETD matches or
outperforms SVGD, AGF-SVGD, and SGLD, with the largest gains in
high-dimensional and multimodal settings.
\end{abstract}

\section{Introduction}
Approximate sampling from an intractable target distribution $\pi(x) \propto \exp(-V(x))$ is a core computational task in Bayesian inference~\citep{Gelman95}, generative modeling~\citep{Kingma21}, and scientific simulation~\citep{VonToussaint11}. Particle-based variational inference (ParVI) methods are a non-parametric approach to variational inference (VI) that addresses this task by evolving an ensemble of $N$ interacting particles such that their empirical distribution converges to $\pi$. Unlike parametric methods, ParVI methods make no restrictive assumptions on the form of the approximating distribution, and unlike Markov chain Monte Carlo (MCMC), they are typically parallelizable, do not require long burn-in periods on multimodal targets to mix, and  yield an \emph{ensemble} of samples.


This paper introduces \emph{entropic transport descent} (ETD),
a family of ParVI algorithms whose inter-particle interaction is
mediated by entropic optimal transport.  Building on the JKO
proximal scheme~\citep{Jordan98}, each iteration solves an
entropic OT problem that assigns every particle a conditional
distribution over a shared pool of target-weighted proposals.
A single parameter $\tau \geq 0$ controls coupling fidelity,
interpolating from a closed-form solution to one that enforces
the target marginal exactly.  ETD can operate score-free---requiring
only pointwise evaluations of the unnormalized target---or
incorporate score information through the proposal distribution.

\paragraph{Contributions.}
(\emph{i})~We introduce ETD, a family of ParVI algorithms
derived from the JKO proximal scheme~\citep{Jordan98} by
lifting the optimization to the space of couplings and
relaxing via the KL chain rule, reducing each iteration
to an entropic OT problem solvable by the Sinkhorn
algorithm (\Cref{sec:etd}).
(\emph{ii})~We characterize the stationary distribution of
balanced ETD exactly (\Cref{thm:stationary}), showing the bias is independent of both the transport cost and the entropic regularization, and provide importance-corrected
target weights that make $\pi$ exactly stationary (\Cref{thm:debiasing}).
(\emph{iii})~Experiments on variance-collapse diagnostics, Bayesian logistic regression, Bayesian neural networks, and sampling from multimodal energy functions demonstrate that ETD matches or outperforms existing ParVI baselines, including both score-based and score-free methods, with the largest gains in high-dimensional and multimodal settings
(\Cref{sec:experiments}).
\section{Background}
\label{sec:background}

\paragraph{Definitions and Notation.} Let $\P(\R^n)$ denote the space of \emph{probability distributions} on $\R^n$ absolutely continuous with respect to Lebesgue measure; all distributions are assumed to have finite second moments. No distinction is made between a probability measure and its density. The \emph{Kullback--Leibler} (KL) \emph{divergence} between $\mu, \pi \in \P(\R^n)$ (with $\mu$ absolutely continuous with respect to $\pi$) is denoted $\KL{\mu}{\pi}$. We write $\delta_x$ for the \emph{Dirac measure} centered at $x$, $\Delta^N$ for the \emph{probability simplex} in $\R^N$, $\mu \otimes \nu$ for the \emph{product measure}, and $\oslash$ for \emph{elementwise division}. The inner product between two matrices $A, B \in \R^{n \times m}$ is denoted: \mbox{$\langle A, B \rangle = \sum_{ij} A_{ij} B_{ij}$}. 
The vector $\mathbf{1}_n \in \mathbb{R}^n$ is the vector for which all entries are one. 
A \emph{coupling} of distributions $\mu, \nu \in \P(\R^n)$ is a joint distribution $\gamma$ on $\R^n \times \R^n$ with marginals $\mu$ and $\nu$; we denote the set of all such couplings $\C[\mu, \nu]$. When only the source marginal is fixed, we write $\C(\mu) \define \bigcup_{\nu \in \P(\R^n)} \C[\mu, \nu]$.

\paragraph{Optimal Transport.}
Optimal transport (OT) provides a geometry on probability
measures by quantifying the cost of moving mass between
distributions. Given a transport cost
$c : \R^n \times \R^n \to \R_{\geq 0}$, the OT cost between
$\mu, \nu \in \P(\R^n)$ is
$\T_c[\mu, \nu] = \inf_{\gamma \in \C[\mu, \nu]}
\int c\, \dd\gamma$, where the infimum ranges over all
couplings with marginals $\mu$ and $\nu$. The \emph{2-Wasserstein distance} $\W_2[\mu, \nu]$ corresponds to $\sqrt{\T_c[\mu, \nu]}$ with
$c(x,y) = \tfrac{1}{2}\|x - y\|^2$. When both marginal constraints are enforced, the problem is
\emph{balanced}; when one or both are relaxed to KL penalties,
it is \emph{unbalanced}; and when only the source marginal is
constrained while the target enters as a reference measure, it
is \emph{semi-relaxed}. Adding an entropic regularization term
$\varepsilon\, \KL{\gamma}{\mu \otimes \nu}$ to the objective
yields \emph{entropic} OT (EOT), which is strictly convex
and solvable via the Sinkhorn algorithm~\citep{Cuturi13}, and whose value is denoted $\T_c^{\varepsilon}[\mu, \nu]$.
Background on OT and EOT, its discrete specialization, and the
Sinkhorn algorithm are provided in
Appendix~\ref{app:ot}.

\paragraph{Variational Inference.} Many statistical methods, e.g., Bayesian inference, require sampling from an intractable \emph{target distribution} $\pi(x) \propto \exp(-V(x))$. In such cases, \emph{variational inference} (VI) is a popular approach to approximately sample from $\pi$ by identifying a family of \emph{tractable} distributions, $\Q \subset \P(\R^n)$, then identifying the closest member of the family to $\pi$ as measured by the KL divergence \citep{Blei17}:
\begin{align}
    \mu^\star\ \define\ \argmin_{\mu \in \Q} \quad \KL{\mu}{\pi}. \tag{VI}\label{eq:vi}
\end{align}
Particle VI (ParVI) methods are a non-parametric approach that selects the tractable family  $\Q$ to be represented by a finite number of interacting particles $\mathbf{z} = (z_i)_{i = 1}^N$---thereby reducing an infinite-dimensional optimization problem over distributions to a finite one. Typically, as in the sequel, particles are treated as individual samples and the tractable family consists of the \emph{empirical distributions}: \mbox{$\mu(x; \mathbf{z}) = \frac{1}{N} \sum_{i = 1}^N \delta_{z_i}(x)$}. Particles are then evolved according to an update law designed so that their mutual interactions drive them to be approximately distributed according to $\pi$. Thus, the choice of interaction mechanism defines many of a ParVI algorithm's properties.


\paragraph{The JKO Scheme.} The probability density $\mu_t \in \P(\R^n)$ of the \emph{overdamped Langevin equation} (OLE) driven by a standard Wiener process $W_t$,
\begin{align}
    \dd X_t = -\partial_x V(X_t)\, \dd t + \sqrt{2}\, \dd W_t, \label{eq:langevin}
\end{align}
is the backbone of many sampling and inference algorithms. Under mild
regularity conditions, $\mu_t$ evolves according to the Fokker--Planck
equation (FPE),
\begin{align}
    \partial_t \mu_t = \partial_x \cdot (\mu_t \partial V) + \partial_x \cdot \partial_x \mu_t,
    \label{eq:fpe}
\end{align}
and converges to the target density
$\pi(x) \propto \exp(-V(x))$ as $t \to \infty$. The \emph{score function} is the gradient of the target log-density: $\partial_x\log \pi(x) = -\partial_x V(x)$.

Jordan, Kinderlehrer, and Otto~\citep{Jordan98} showed that
\eqref{eq:fpe} is the gradient flow of $\KL{\cdot}{\pi}$ with
respect to $\W_2$ on $\P(\R^n)$. This structure yields
a natural implicit discretization---the \emph{JKO scheme}:
\begin{align}
    \mu_{k+1} = \argmin_{\mu \in \P(\R^n)} \quad
    \KL{\mu}{\pi} \;+\; \frac{1}{2\tau}\, \W_{2}^2[\mu, \mu_k],
    \label{eq:jko}
\end{align}
where $\tau > 0$ is the step size. Each step advances $\mu_{k+1}$
toward $\pi$ while penalizing large displacements from
$\mu_k$; as $\tau \to 0$, the iterates recover~\eqref{eq:fpe}.

The JKO scheme is elegant but computationally demanding---each
iteration requires solving an optimal transport problem. Even in simple discrete settings, solving such problems has a complexity $O(N^3 \log N)$.
\Cref{sec:etd} develops a tractable upper bound
on~\eqref{eq:jko} by introducing entropic regularization, which produces a strictly convex transport problem that is solved via the Sinkhorn algorithm~\citep{Cuturi13}.

\begin{figure}[t]
    \centering
    \includegraphics[width=\textwidth]{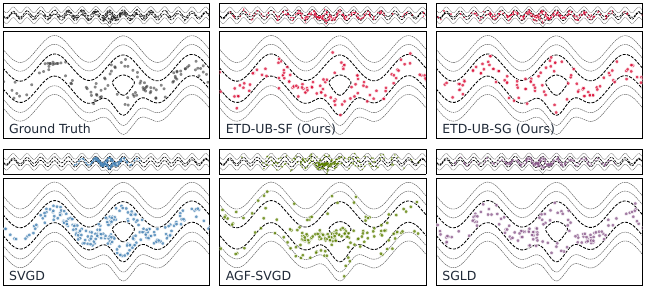}
    \caption{\small Approximate samples from the U3 energy potential
  of \citet{Rezende15} ($N = 250$ particles, 5 seeds pooled).
  Top strips show a wide view; bottom panels focus on the central
  multimodal region.  SVGD and AGF-SVGD both suffer from variance
  collapse.  SGLD covers both modes but provides sparser coverage
  of the target structure with the same particle budget.  Both
  score-guided (ETD-UB-SG) and score-free (ETD-UB-SF) variants of
  ETD capture the bimodal structure and match the dispersion of the
  target, outperforming all baselines in energy distance
  (Appendix~\ref{app:experiment_setup}).}
    \label{fig:u3}
    \vspace{-1em}
\end{figure}

\section{Related Work}
\label{sec:related work}
%

\paragraph{Particle-Based Variational Inference.}
SVGD~\citep{Liu16} evolves particles along a kernelized Stein gradient flow in an RKHS.
While foundational, its kernel-mediated repulsion weakens in high
dimensions, leading to variance collapse~\citep{Zhuo18,DAngelo21} and mode
collapse on multimodal targets.
\citet{Liu19} and \citet{Chewi20} clarified the Wasserstein gradient flow
structure underlying SVGD, with LAWGD providing uniform ergodicity guarantees.
Subsequent methods replace kernel repulsion with alternative interaction mechanisms:
mollified interaction energies~\citep{Li23}, blob
methods from PDE numerics~\citep{Craig16,Craig23}, and
semi-discrete OT~\citep{Ambrogioni18}.
ETD departs from these approaches by mediating particle interactions through
entropic transport couplings---each particle's update depends on the full
ensemble via a shared transport plan---and by supporting fully score-free operation.

\paragraph{Wasserstein Gradient Flows and JKO Schemes.}
Neural JKO methods~\citep{Mokrov21,AlvarezMelis22,Fan22,Bunne22}
make proximal Wasserstein steps tractable by training a network per
iteration. \citet{Peyre15} introduced the entropic JKO scheme, with
convergence established by \citet{Carlier17}.
The iterated Schr\"{o}dinger bridge method~\citep{Agarwal24} is the
closest antecedent: it approximates JKO steps via entropic OT
without scores or neural networks, but requires target samples and
uses a single coupling regime. ETD instead generates local proposals
from pointwise evaluations of $\pi$ and provides the $\tau$-family
(Proposition~\ref{prop:coupling}) trading cost for coupling fidelity.
\citet{Lambert22} developed Bures--Wasserstein VI for Gaussian
families; \citet{Bonet24} analyzed preconditioned Wasserstein
gradient descent; and \citet{Hardion25} studied gradient flows in
Sinkhorn divergence geometry.

\paragraph{Entropic Optimal Transport.}
ETD's coupling step builds on the Sinkhorn algorithm~\citep{Cuturi13}, with
log-domain stabilization~\citep{Schmitzer19} enabling numerically stable
computation at small $\varepsilon$.
The Sinkhorn divergence~\citep{Feydy19,Genevay18} debiases entropic OT by
subtracting self-transport terms;
for balanced ETD, we instead eliminate bias via importance-corrected target
weights (\Cref{thm:debiasing}).
The unbalanced formulations underlying ETD's $\tau$-family extend the scaling
algorithms of \citet{Chizat18} and the
framework of \citet{Sejourne22}.


\paragraph{Transport-Based Sampling and Generative Modeling.}
Deterministic transport maps for Bayesian
inference~\citep{ElMoselhy12,Spantini18} avoid MCMC but require learning
triangular maps.
Flow matching~\citep{Lipman23,Tong24} and Schr\"{o}dinger bridge
methods~\citep{DeBortoli21,Vargas24} solve dynamic OT problems on path space,
achieving strong results in high-dimensional generative modeling.
These methods amortize computation through neural networks trained across
queries.
ETD operates in the finite-particle regime without neural network inner loops,
trading amortization for applicability when the target changes at every
iteration (e.g., stochastic optimal control~\citep{Pacelli26}) or when
computation is limited.

\section{Variational Inference via Entropic Transport}
\label{sec:etd}
This section develops a family of tractable approximations to the JKO scheme, i.e., the ETD methods. After lifting the intractable optimization over distributions to the space of couplings, the objective is approximated by bounding the KL divergence. The result is a family of EOT problems, parameterized by a single scalar $\tau \geq 0$, the solutions of which are computed using the Sinkhorn algorithm---and form the updates for our particle method. 


\paragraph{Transport-Regularized Variational Inference.}
\label{sec:trvi}

Given a current approximation $\mu \in \P(\R^n)$, the JKO
scheme~\eqref{eq:jko} produces the next iterate by minimizing
the KL divergence to the target subject to a transport penalty. We write this:
\begin{align}
    \minimize_{\nu \in \P(\R^n)} \quad
    \KL{\nu}{\pi} \;+\; \frac{1}{\varepsilon}\,
    \T_c[\nu, \mu],\label{eq:trvi}
\end{align}
where $\varepsilon^{-1} > 0$ plays the role of a Lagrange
multiplier (equivalently, $\varepsilon$ is a step-size parameter).
Iterating~\eqref{eq:trvi} produces a sequence of distributions
converging to $\pi$; as $\varepsilon \to 0$, the iterates recover
the JKO scheme with step size $\tau = \varepsilon$.

Problem~\eqref{eq:trvi} is a principled starting point---it
inherits the convergence guarantees of the Wasserstein gradient
flow---but it requires optimizing over distributions, which is
generally intractable. The transport cost $\T_c[\nu, \mu]$ itself
involves an inner optimization over couplings,
compounding the difficulty. We address both issues simultaneously
by lifting the optimization to the space of couplings.

\paragraph{Lifting to Couplings.}
\label{sec:lifting}

Any coupling $\gamma \in \P(\R^n \times \R^n)$ with source marginal $\mu$ induces a target marginal $\nu_\gamma(\cdot) = \gamma(\R^n \times \cdot)$. By optimizing over $\gamma \in \C(\mu)$ rather than $\nu$ directly, the transport cost is absorbed into the coupling
and the inner optimization in $\T_c$ disappears. Problem~\eqref{eq:trvi} becomes
\begin{align}
    \minimize_{\gamma \in \C(\mu)} \quad
    \KL{\nu_\gamma}{\pi} \;+\; \frac{1}{\varepsilon}
    \int c(x, y)\, \dd\gamma(x, y).
    \label{eq:lifted}
\end{align}
This formulation optimizes simultaneously over the next
distribution \emph{and} the transport plan to reach it.
However, the KL term still couples all components of $\gamma$
through the induced marginal $\nu_\gamma$, limiting tractability.
The next section resolves this via a variational relaxation.

\paragraph{Relaxation via the Chain Rule}
\label{sec:relaxation}

The main step in developing a tractable algorithm is decomposing the KL divergence into a coupling and a product reference measure using the KL chain rule (proof in Appendix~\ref{app:proofs}).

\begin{restatable}[Chain Rule for Couplings]{proposition}{chainrule}
\label{lem:chainrule}
Let $\gamma \in \C(\mu)$ and denote its induced target marginal
$\nu_\gamma(\cdot) = \gamma(\R^n \times \cdot)$. Then, for $\pi \in \P(\R^n)$:
\begin{align}
    \KL{\gamma}{\mu \otimes \pi}
    \;=\; \KL{\gamma}{\mu \otimes \nu_\gamma}
    \;+\; \KL{\nu_\gamma}{\pi}.
    \label{eq:chainrule}
\end{align}
\end{restatable}



The mutual information satisfies
$I(\gamma) \define \KL{\gamma}{\mu \otimes \nu_\gamma} \geq 0$,
so $\KL{\nu_\gamma}{\pi} \leq
\KL{\gamma}{\mu \otimes \pi}$. Substituting this upper bound
into~\eqref{eq:lifted} produces a tractable surrogate objective.

\begin{restatable}[ETD Objective]{theorem}{etd}
\label{prop:etd}
The optimization problem
\begin{align}
    \minimize_{\gamma \in \C(\mu)} \quad
    \int c\, \dd\gamma
    \;+\; \varepsilon\, \KL{\gamma}{\mu \otimes \pi}
    \tag{ETD}\label{eq:etd}
\end{align}
is an upper bound on the lifted problem~\eqref{eq:lifted} in the
following sense: for any $\gamma \in \C(\mu)$,
\begin{align}
    \KL{\nu_\gamma}{\pi}
    + \frac{1}{\varepsilon} \int c\, \dd\gamma
    \;\leq\;
    \frac{1}{\varepsilon}\left(
    \int c\, \dd\gamma
    + \varepsilon\, \KL{\gamma}{\mu \otimes \pi}
    \right),
    \label{eq:upperbound}
\end{align}
with equality if and only if $\gamma = \mu \otimes \nu_\gamma$
(i.e., the coupling is independent). In particular:
\begin{enumerate}[leftmargin=2.0em, label=(\roman*), itemsep=0.1em, parsep=0.1em, topsep=0.0em, partopsep=0.0em]
    \item The gap between~\eqref{eq:lifted}
    and~\eqref{eq:etd} at their respective minimizers is bounded
    by the mutual information $I(\gamma^\star)$ of the
    \eqref{eq:etd}-optimal coupling.
    \item As $\varepsilon \to \infty$, $\gamma^\star$ approaches independence ($I(\gamma^\star) \to 0$) and the bound becomes tight.
    \item Problem~\eqref{eq:etd} is an instance of semi-relaxed
    entropic optimal transport: the source marginal $\mu$ is
    constrained while the target $\pi$ enters as the reference
    measure in the KL regularizer.
\end{enumerate}
\end{restatable}

The semi-relaxed structure of~\eqref{eq:etd} is what makes it
tractable: because the target enters as a reference measure rather
than a marginal constraint, the optimal coupling admits a
closed-form Gibbs solution (Section~\ref{sec:algorithm}). Note
that the target $\pi$ need only be evaluated pointwise up to a
normalizing constant---no samples from $\pi$ are required.

\subsection{A Family of Relaxations}
\label{sec:tau-family}

Theorem~\ref{prop:etd} establishes~\eqref{eq:etd} as an upper
bound on the lifted problem~\eqref{eq:lifted}, with a gap
equal to the mutual information $I(\gamma)$. More generally, for any $\tau \geq 0$, the objective
\begin{align}
    \minimize_{\gamma \in \C(\mu)} \quad
    \int c\, \dd\gamma
    \;+\; \varepsilon\, \KL{\gamma}{\mu \otimes \pi}
    \;+\; \tau\, \KL{\nu_\gamma}{\pi}
    \tag{$\tau$-ETD}\label{eq:etd-tau}
\end{align}
is also an upper bound on~\eqref{eq:lifted}. To see this,
apply the chain rule (\Cref{lem:chainrule}) to
decompose the KL term:
\begin{align}
    \int c\, \dd\gamma
    \;+\; \varepsilon\, I(\gamma)
    \;+\; (\varepsilon + \tau)\, \KL{\nu_\gamma}{\pi}.
    \label{eq:etd-tau-decomposed}
\end{align}
This exceeds the lifted objective
$\int c\, \dd\gamma + \varepsilon\, \KL{\nu_\gamma}{\pi}$
by $\varepsilon\, I(\gamma) + \tau\, \KL{\nu_\gamma}{\pi} \geq 0$
for all $\gamma \in \C(\mu)$ and all $\tau \geq 0$. The family varies continuously in $\tau$: because $\pi$ appears
in the reference measure for all $\tau \geq 0$, the unbalanced
solution interpolates between the semi-relaxed ($\tau \to 0$) and
balanced ($\tau \to \infty$) extremes. Larger $\tau$ enforces the
target marginal more strongly, producing a better minimizer at the
cost of a harder optimization. In the sequel, \Cref{prop:coupling} describes three distinct regimes of $\tau$.

\subsection{The ETD Algorithm}
\label{sec:algorithm}

\begin{wrapfigure}{r}{0.52\columnwidth}
  \vspace{-1.8em}
  \begin{minipage}{0.52\columnwidth}
  \input{algorithms/etd}
  \end{minipage}
  \vspace{-1.0em}
\end{wrapfigure}

We now specialize the framework to empirical measures and develop
the practical algorithm. Represent the current distribution by $N$
particles $\mu = \frac{1}{N}\sum_{i=1}^N \delta_{x_i}$ with
uniform weights $a_i = 1/N$. To solve~\eqref{eq:etd-tau}, we
require a second empirical measure to transport toward. Unlike
most OT-based inference methods~\citep{Agarwal24,Bunne22}, ETD
does not require samples from the target---only pointwise
evaluations of the unnormalized density $\pi$. Instead, we
generate a set of $M$ \emph{proposal} positions
$\{y_j\}_{j=1}^M$ and assign them target weights
$b_j \propto \pi(y_j)$, $b \in \Delta^M$. The proposals provide
the support of the next iterate; the coupling determines how
particles are redistributed onto this support.

Each ETD iteration then proceeds in three steps: \emph{propose}
candidate positions, \emph{couple} particles to proposals via
entropic OT, and \emph{update} particles according to the
coupling. We describe each in turn.

\paragraph{Propose.}
\label{sec:proposals}
The framework is agnostic to the proposal
mechanism---any method that produces candidate positions
$\{y_j\}_{j=1}^M$ is compatible. However,  proposals should cover regions of high target density near the
current particles. A natural choice of sampling distribution $q(y \mid x_i)$ is a single Euler--Maruyama
step of the OLE~\eqref{eq:langevin}:
\begin{align}
    y_{ij} = x_i + \alpha\, \partial_x \log \pi(x_i)
    + \sigma \xi_{ij}, \quad
    \xi_{ij} \sim \mathcal{N}(0, I), 
    \label{eq:proposals}
\end{align}
where $\alpha$ is a score step size and $\sigma$ a noise scale. Optionally, a \emph{coupling momentum} term can be included:
$y_{ij} = x_i + \alpha\, \partial_x \log \pi(x_i)
+ \mu\, d_i + \sigma\xi_{ij}$,
where $d_i = \sum_j \gamma_{ij} y_j - x_i$ is the barycentric
displacement from the previous iteration and $\mu \in [0, 1)$
is a momentum coefficient. The score biases proposals toward high-density regions; the noise provides exploration. If the score is unavailable, setting
$\alpha = 0$ yields a random-walk proposal
$y_{ij} = x_i + \sigma \xi_{ij}$; the coupling is responsible for guiding particles via the target weights
$b_j \propto \pi(y_j)$. Define the \emph{pooled proposal density} as the mixture of per-particle proposal densities $q_\mu(y) := \frac{1}{N}\sum_{i=1}^{N} q(y \mid x_i)$. 

\paragraph{Couple.}
The discrete specialization of the $\tau$-family~\eqref{eq:etd-tau}
yields a coupling $\Gamma \in \R_+^{N \times M}$ between particles
and proposals:
\begin{align}
    \minimize_{\Gamma \geq 0,\; \Gamma \one_M = a} \quad
    \langle C, \Gamma \rangle
    + \varepsilon\, \KL{\Gamma}{a \otimes b}
    + \tau\, \KL{\nu_\Gamma}{b},
    \tag{$\tau$-EOT}\label{eq:eot-tau}
\end{align}
where
$\nu_\Gamma = \Gamma^\top \one_N$ is the induced target marginal and $C_{ij} = c(x_i, y_j)$ is the cost matrix, normalized to unit median. 

We consider two transport costs: the squared Euclidean distance and a diagonal Mahalanobis cost $c(x,y) = \tfrac{1}{2}(x-y)^\top \hat\Sigma^{-1}(x-y)$, where $\hat\Sigma = \mathrm{diag}(\mathrm{Var}_i[x_i])$ is recomputed from the particle ensemble at each iteration. The Mahalanobis cost rescales coordinates by their ensemble spread, improving conditioning of the Gibbs kernel on anisotropic targets.
For all three regimes, the optimal coupling has the Gibbs form
$\Gamma^\star_{ij} = u_i K_{ij} v_j$ with
$K_{ij} = \exp(-C_{ij}/\varepsilon)$. The regimes differ only in
how the scaling vector $v$ is determined (proof in Appendix~\ref{app:proofs}):
\begin{restatable}[Coupling Types]{proposition}{coupling}
\label{prop:coupling}
The optimal coupling of~\eqref{eq:eot-tau} satisfies
$\Gamma^\star_{ij} = u_i K_{ij} v_j$ with
$u_i = a_i / (Kv)_i$ and:
\begin{enumerate}[leftmargin=2.0em, label=(\roman*), itemsep=0.1em, parsep=0.1em, topsep=0.0em, partopsep=0.0em]
    \item \textbf{Semi-relaxed} ($\tau = 0$):
    $v_j = b_j$.
    \item \textbf{Unbalanced} ($0 < \tau < \infty$):
    $v_j = b_j / (K^\top u)_j^{\rho/(1+\rho)}$ with
    $\rho \define \tau/\varepsilon$ (Algorithm~\ref{alg:sinkhorn}).
    \item \textbf{Balanced} ($\tau \to \infty$):
    $v_j = b_j / (K^\top u)_j$ (Algorithm~\ref{alg:sinkhorn}).
\end{enumerate}
\end{restatable}
The row-normalized coupling $\gamma_{ij} \define
\Gamma^\star_{ij} / a_i$ defines a conditional distribution over
proposals for each particle $i$, denoted $\gamma_{i\cdot} \in \Delta^M$.

\paragraph{Update.}
Each particle samples its next position from the conditional
coupling:
\begin{align}
    x_i^{\mathrm{new}} = y_j \quad \text{with probability}
    \quad \gamma_{ij}.
    \label{eq:categorical}
\end{align}
This \emph{categorical update} commits each particle to a single proposal, preserving diversity and maintaining mode fidelity on multimodal targets.

\section{Theoretical Analysis}
\label{sec:theory}
We analyze the stationary behavior of balanced ETD, the variant
whose marginal constraint provides the strongest inter-particle
coordination. We first characterize the stationary distribution
of the balanced chain with standard target weights, diagnosing
the source and structure of the bias. We then show the bias can be eliminated via importance-corrected weights. Unless stated, we assume $\alpha = 0$, balanced couplings ($\tau \to \infty$), and the population limit. Proofs are deferred to Appendix~\ref{app:proofs}.

\paragraph{Stationary Distribution.} The case of balanced coupling enforces the target marginal exactly:
$\int \gamma(y \mid x)\, \mu(\dd x) = \beta_\mu(y)$, where
$\beta_\mu(y) \propto \pi(y)\, q_\mu(y)$ is the effective
target marginal.
After the categorical update, the aggregate law of the
new ensemble is $\beta_\mu$ regardless of $\varepsilon$---the
marginal constraint fixes the output while $\varepsilon$
controls only how individual particles are assigned within
that output. Stationarity requires $\mu^* = \beta_{\mu^*}$,
yielding a self-consistency condition.

\begin{restatable}[Stationary Distribution]{theorem}{stationary}
\label{thm:stationary}
Consider score-free ETD ($\alpha = 0$) with balanced coupling
($\tau \to \infty$), categorical update, and the population
limit $M, N \to \infty$. A distribution $\mu^*$ is stationary
for the mean-field dynamics if and only if
$\mu^*(y) \propto
    \pi(y) \cdot [\mu^* * G_{\sigma^2}](y)$,
where $G_{\sigma^2}$ is the isotropic Gaussian kernel with
bandwidth $\sigma^2$. Moreover: (i) The stationary distribution is independent of the choice of transport cost $c(x, y)$ and entropic regularization $\varepsilon$ and (ii) For any target $\pi$ with finite second moment and any $\sigma > 0$, at least one such $\mu^*$ exists.
\end{restatable}

The ratio $\mu^*(y)/\pi(y) \propto [\mu^* * G_{\sigma^2}](y)$
identifies the bias: the stationary distribution over-weights
regions where $\mu^*$ has high density within an
$O(\sigma)$-neighborhood, producing systematic
over-concentration relative to $\pi$. The independence from
$c$ and $\varepsilon$ is practically significant: these
parameters affect the mixing rate of the finite-$N$ chain but
not the asymptotic target, which depends only on $\pi$ and the
proposal bandwidth~$\sigma$.

\paragraph{Bias Correction.} The bias in \Cref{thm:stationary} arises because standard
weights $b_j \propto \pi(y_j)$ double-count the proposal
density: proposals are already concentrated near high-density
regions, and weighting by $\pi$ reinforces this concentration.
Dividing by $q_\mu$ removes the redundancy.

\begin{restatable}[Debiasing ETD]{theorem}{debiasing}
\label{thm:debiasing} 
Under the conditions of Theorem~\ref{thm:stationary} and for any non-negative cost $c$ and $\varepsilon > 0$, define the importance-corrected target weights
$b_j^{\mathrm{IC}} \propto \pi(y_j) / q_\mu(y_j)$.
Then, (i)~$\pi$ is a stationary distribution of the
resulting chain and (ii)~if the cost is symmetric, $c(x,y) = c(y,x)$, the chain is $\pi$-reversible.
\end{restatable}


For the semi-relaxed coupling ($\tau = 0$), where the
conditional is a single-particle kernel, the bias admits a
closed-form characterization
$\pi^*_\varepsilon(x) \propto
\pi(x) \cdot [\pi * G_{\varepsilon_\mathrm{eff}}](x)$
and can be eliminated via a complementary Metropolis--Hastings step with
acceptance ratio
$\alpha_{\mathrm{MH}}(x,y) = \min(1, Z(x)/Z(y))$, where
$Z(x)$ is precomputed during the coupling step
(Appendix~\ref{app:proofs}, Eq. \ref{eq:mh-ratio}).

\section{Experiments}
\label{sec:experiments}
\renewcommand{\thefootnote}{\textdagger}
\begin{table}[t]
  \centering
  \caption{Summary of main results. For ETD, each column reports the best-performing variant; see Appendix~\ref{app:experiment_setup} for the full description, parameters, and statistics.
    \textbf{DAMV.} dimension-averaged marginal variance (1.0 = exact recovery).  
    \textbf{Cov90.} 90\% marginal coverage against NUTS reference (nominal 0.9).  
    \textbf{NLL Best.} datasets (of 9) on which the method achieves the lowest test NLL.  
    \textbf{TV.} total variation of pairwise-distance histograms vs.\ MCMC reference.}
  \label{tab:summary}
  \small
  \begin{tabular*}{\columnwidth}{@{\extracolsep{\fill}}lccccr@{}}
    \toprule
    & \multicolumn{2}{c}{Variance Collapse (DAMV)} 
    & {Covertype} 
    & {BNN UCI} 
    & {LJ-13} \\
    \cmidrule(r){2-3} \cmidrule(lr){4-4} \cmidrule(lr){5-5} \cmidrule(l){6-6}
    {Method}
    & {$d{=}50$}
    & {$d{=}200$}
    & {Cov90 $\uparrow$}
    & {NLL best (/9)}
    & {TV $\downarrow$} \\
    \midrule
    ETD (ours)\footnotemark
      & \textbf{1.00}
      & \textbf{1.00}
      & \textbf{0.986}
      & \textbf{6}
      & \textbf{0.053} \\
    SVGD
      & 0.16
      & 0.11
      & 0.323
      & 3
      & 0.39 \\
    SGLD
      & 0.98
      & 0.97
      & 0.277
      & 0
      & {0.42} \\
    AGF-SVGD
      & 3.92
      & 3.97
      & 0.201
      & 0
      & {0.39} \\
    \bottomrule
  \end{tabular*}
\end{table}
\footnotetext{Best ETD variant per benchmark: \textbf{DAMV $d{=}50$.} ETD-BAL-Euc (IS), score-guided; 
\textbf{DAMV $d{=}200$.} ETD-BAL-Mom (IS), score-guided; 
\textbf{Covertype.} ETD-SR-Maha, score-guided; 
\textbf{BNN.} ETD-BAL-Mom, \emph{score-free}; 
\textbf{LJ-13.} ETD-SR-Maha, \emph{score-guided}.}
\renewcommand{\thefootnote}{\arabic{footnote}}

\textbf{Baselines.} We compare ETD against SVGD \citep{Liu16}, SGLD \citep{Welling11}, and annealed gradient-free SVGD (AGF-SVGD) \citep{han2018stein}. SVGD follows the standard \citep{Liu16} implementation using an RBF kernel, the median heuristic, and the AdaGrad optimizer. For SGLD we use the parallel version that runs $N$ independent chains with a decaying step size and takes the final states as the particle approximation. Since ETD can operate in score-free mode, we also include AGF-SVGD as a score-free baseline that replaces the score function with importance-weighted kernel density surrogates.

We evaluate ETD on a diverse set of synthetic and real-data benchmarks to assess distributional accuracy, mode coverage, variance preservation, and scalability. For a fair comparison, we tune all ETD variants and baselines with Optuna under the same benchmark-specific objective. Tuning objectives, hyperparameters, implementation details, and results are provided in Appendix~\ref{app:experiment_setup}.

\paragraph{Variance Collapse.}

Variance collapse is a well-documented failure mode of deterministic particle methods such as SVGD, where particle approximations underestimate the target variance in high dimensions. To isolate this effect, we use the isotropic Gaussian target $\mathcal{N}(0, I_d)$ with particles initialized from $\mathcal{N}(2 \cdot \mathbf{1}_d, 4I_d)$, requiring methods to correct the mean as well as recover the target spread.
All methods use $N=50$ particles and are tuned at $d=50$ before being evaluated at $d \in \{10, 20, 50, 100, 200\}$.
We measure the dimension-averaged marginal variance (DAMV); for this experiment, $\mathrm{DAMV}=1.0$ corresponds to exact variance recovery.

SVGD collapses rapidly as dimension increases, from $\mathrm{DAMV} = 0.415$ at $d=10$ to $\mathrm{DAMV} = 0.110$ at $d=200$. AGF-SVGD fails in the opposite direction, inflating variance by roughly $4\times$ across all dimensions. SGLD maintains stable variance throughout, as expected from a Langevin sampler. The best score-guided ETD variants achieve near-perfect variance recovery at all dimensions, matching or improving on SGLD while avoiding the collapse observed in SVGD. Score-free ETD variants maintain accurate variance recovery through $d=100$ and degrade only at $d=200$. Per-dimensional results are shown in Table~\ref{tab-app:vc_results}.

\paragraph{2D Energy Functions.}

We use the four two-dimensional energy-function targets adapted from \citet{Rezende15}---$U_1$ (ring with bumps), $U_2$ (sinusoidal banana), $U_3$ (parallel sinusoidal modes), and $U_4$ (sigmoid-offset modes)---to evaluate how well each method captures complex low-dimensional structure. We additionally include an 8-mode ring Gaussian mixture as a mode-coverage stress test. Unlike the other benchmarks, methods are tuned per-target with 50 Optuna trials. We evaluate using energy distance against 10K ground-truth samples, averaged over 5 seeds.

On $U_1$ and the ring GMM, many ETD variants achieve excellent distributional coverage ($|E_{\text{dist}}| < 0.05$). On the harder multimodal targets $U_2$--$U_4$, importance sampling correction becomes critical: the top-performing ETD methods consistently use IS correction, and unbalanced coupling achieves the best (or tied-for-best) energy distance on most targets. SVGD performs poorly on the multimodal targets ($U_2$--$U_4$), exhibiting severe mode collapse. Figure~\ref{fig:u3} illustrates this on $U_3$---the best ETD variant captures both sinusoidal modes, while SVGD and AGF-SVGD fail to recover the full target structure. Per-target results are provided in Table~\ref{tab-app:viz2d_results}.

\paragraph{Bayesian Logistic Regression.}

We evaluate ETD on standard Bayesian logistic regression (BLR) with a Gaussian prior and Gamma hyperprior. We test on two real-data benchmarks: German Credit ($d=26$) and Covertype ($d=56$). For all BLR experiments, ETD and baselines are tuned with Optuna before final evaluation. German Credit uses a 90/10 train/test split, while Covertype uses 70/10/20 train/validation/test with validation-based tuning. All methods use $N=100$ particles; we evaluate over 20 random splits and report predictive negative log-likelihood (NLL) and classification accuracy.
For Covertype, we also report posterior diagnostics including energy distance and 90\% marginal coverage against a 20K-sample No U-Turn Sampler (NUTS) reference posterior.

\textit{German Credit.} German Credit is a small, well-conditioned dataset, and the best ETD variants and all baselines (except AGF-SVGD) converge to essentially the same predictive performance---within 0.001 NLL of each other, with no statistically significant differences. Accuracy shows a similar pattern, clustering around 77--78\%. This saturation is expected at $d=26$ with $n=900$: the posterior is sufficiently simple that all methods recover it. German Credit thus serves as a sanity check, confirming that ETD recovers the correct posterior relative to NUTS. Results are shown in Table~\ref{tab-app:blr_german_results}.

\textit{Covertype.} This dataset presents a substantially harder inference problem with approximately $n=407$K training samples, $d=56$ parameters (after adding a bias term and log-precision parameters), and minibatch likelihood evaluation. We use 20 random 70/10/20 train/validation/test splits. In terms of NLL, the best ETD variants and baselines perform comparably (0.514--0.515): SVGD ($\mathrm{NLL}=0.5144$) and SGLD ($\mathrm{NLL}=0.5145$) achieve the lowest point estimates, and the best ETD variant ($\mathrm{NLL}=0.5147$) is within $0.001$ of both. However, posterior diagnostics reveal that SVGD, SGLD, and AGF-SVGD all exhibit severe under-coverage. Both SVGD and SGLD have low marginal coverage relative to the NUTS reference, while the best score-guided ETD-SR variants achieve much higher coverage ($0.97$--$0.99$) with only a negligible NLL gap. Results are shown in Table~\ref{tab-app:blr_covertype_results}.

\paragraph{Bayesian Neural Networks.} We evaluate ETD on the standard UCI BNN regression benchmark commonly used in prior work on Bayesian neural networks and particle-based variational inference. We use a one-hidden-layer network with 50 ReLU units and place a $\mathrm{Gamma}(1,0.01)$ prior on both the noise precision and weight precision. All baselines and ETD variants are tuned on Boston Housing using Optuna and then applied to all 9 datasets without per-dataset retuning. We use $N=100$ particles and run each method for 2000 iterations. Results are averaged over 20 random 90/10 train/test splits for all datasets except Protein, which uses 5 splits.

Tables~\ref{tab-app:bnn_results_1} and~\ref{tab-app:bnn_results_2} report test NLL and RMSE for the baselines and selected ETD variants. ETD achieves lower NLL than SVGD on 6 of 9 datasets, lower NLL than SGLD on 8 of 9, and lower NLL than AGF-SVGD on all 9. The dominant variant is score-free ETD-BAL with Euclidean cost and momentum, which outperforms the score-guided ETD variants on 6 of 9 datasets without importance sampling correction. Some score-guided variants suffer from instabilities on higher-dimensional BNN posteriors. These results demonstrate that ETD can match score-guided particle methods while operating in score-free mode, and substantially outperforms the score-free AGF-SVGD baseline.

\paragraph{Molecular Boltzmann Distributions.}

We further evaluate ETD on two physics-inspired Boltzmann sampling benchmarks---Double-Well-4 (DW-4, $d=8$) and Lennard-Jones 13 (LJ-13, $d=39$)---introduced by \citet{kohler2020equivariant} and commonly used in equivariant flow and generative modeling literature. For both tasks, we tune all methods using the same objective: total variation (TV) distance between the pairwise-distance histogram and the MCMC reference distribution. We use $N=100$ particles with 2000 iterations for DW-4 and 5000 for LJ-13; metrics are averaged over 20 seeds.

Tables~\ref{tab-app:dw4_results} and~\ref{tab-app:lj13_results} report results for both tasks. On DW-4, ETD is competitive with SGLD and clearly improves over SVGD and AGF-SVGD, although some ETD variants exhibit occasional high-energy outlier seeds and are excluded from the results. On LJ-13, ETD shows substantially stronger performance: the best variants achieve $\mathrm{TV} \approx 0.053$ against the MCMC reference, compared with $0.39$--$0.42$ for SVGD, SGLD, and AGF-SVGD. All three baselines produce numerically divergent energies on all 20 seeds, making ETD the only method that yields physically meaningful samples on this target. These results demonstrate that ETD's transport coupling is especially effective on higher-dimensional structured targets.

\section{Discussion and Conclusion}
\label{sec:discussion}
We introduced entropic transport descent, a family of particle-based variational
inference algorithms in which inter-particle interactions are mediated by entropic
optimal transport.  The framework is derived from the JKO proximal scheme by
lifting to couplings and relaxing via the KL chain rule, yielding a single-parameter
family of EOT problems solvable by the Sinkhorn algorithm.  The balanced variant
admits an exact characterization of its stationary distribution (Theorem~\ref{thm:stationary}),
with the notable property that the bias depends only on the proposal bandwidth
$\sigma$ and is independent of both the transport cost and the entropic
regularization~$\varepsilon$.  Importance-corrected target weights eliminate this
bias entirely (Theorem~\ref{thm:debiasing}), making $\pi$ exactly stationary
without requiring a Metropolis correction step.  Experiments confirm that ETD
matches or outperforms SVGD, SGLD, and AGF-SVGD across a range of benchmarks,
with the largest gains on multimodal targets and in high-dimensional settings
where kernel-based repulsion breaks down.

\paragraph{Limitations.}
ETD's per-iteration cost is $O(NML)$ for $L$ Sinkhorn iterations, compared with
$O(N^2)$ for SVGD and $O(N)$ for SGLD; the balanced coupling in particular
requires iterating to convergence.
More fundamentally, while our experiments demonstrate better dimensional scaling
than SVGD, the quadratic transport cost $c(x,y) = \tfrac{1}{2}\|x-y\|^2$
concentrates the Gibbs kernel $\exp(-c/\varepsilon)$ as dimension grows,
eventually rendering the coupling uninformative---a manifestation of the curse of
dimensionality shared by all methods that rely on Euclidean distances in ambient
space.

\paragraph{Future Work.}
The ETD framework admits \emph{any} non-negative transport cost, and the
balanced stationary distribution is cost-independent
(Theorem~\ref{thm:stationary}), so the choice of cost affects only the mixing
rate, not the asymptotic target.
This separation opens a principled avenue for future work: designing
transport costs and proposal distributions that preserve informative coupling
structure in high dimensions---for example, costs derived from local geometry of
the target (Riemannian metrics, Fisher information), sliced or
projected formulations, or learned cost functions.
Investigating the interplay between cost design and the proposal mechanism may
yield algorithms that scale to the high-dimensional problems
encountered in modern Bayesian deep learning and scientific simulation.

\newpage
\bibliographystyle{plainnat}
{\small
\bibliography{bibliography}
}

\newpage
\appendix
\section{Background on Optimal Transport}
\label{app:ot}
\begin{wrapfigure}{r}{0.45\textwidth}
  \vspace{-\intextsep}                    
  \begin{minipage}{0.45\textwidth}
    \hrule height 0.8pt 
    \smallskip
    \captionof{algorithm}{Sinkhorn Algorithm}
    \label{alg:sinkhorn}
    \vspace{-4pt}
    \hrule 
    \smallskip
    \input{algorithms/sinkhorn}
    \hrule
  \end{minipage}
  \vspace{-\intextsep}                    
\end{wrapfigure}

\paragraph{Optimal Transport.} Given a transport cost function $c : \R^n \times \R^n \to \R_{\geq 0}$, the
\emph{optimal transport} (OT) cost is:
\begin{align}
    \T_c[\mu, \nu] \define
    \inf_{\gamma \in \C[\mu, \nu]} \;
    \int c(x, y)\, \dd\gamma(x, y).
    \tag{OT}\label{eq:ot}
\end{align}
The 2-Wasserstein distance is defined as $\W_2[\mu, \nu] \define \sqrt{\T_c[\mu, \nu]}$ with the cost $c(x,y) = \tfrac{1}{2}\|x - y\|^2$.

\paragraph{Entropic Optimal Transport.}
Even in the discrete setting, solving \eqref{eq:ot} exactly costs
$O(N^3 \log N)$. Entropic regularization~\citep{Cuturi13} introduces
a strictly convex penalty that yields efficient algorithms and smooth
dependence on the inputs:
\begin{align}
    \T_c^\varepsilon[\mu, \nu] \define
    \inf_{\gamma \in \Gamma(\mu, \nu)} \;
    \int c(x, y)\, \dd\gamma(x, y)
    \;+\; \varepsilon\, \KL{\gamma}{\mu \otimes \nu},
    \tag{EOT}\label{eq:eot}
\end{align}
where $\mu \otimes \nu$ is the product measure and
$\varepsilon > 0$ controls the regularization strength. The KL term
penalizes deviation from independence, encouraging smooth transport
plans. As $\varepsilon \to 0$, $\T_c^\varepsilon$ recovers
$\T_c$; for $\varepsilon > 0$, the unique optimum has the Gibbs
form
\begin{align}
    \frac{\dd\gamma^\star}{\dd(\mu \otimes \nu)}(x, y)
    = \exp\!\left(\frac{f(x) + g(y) - c(x,y)}{\varepsilon}\right),
    \label{eq:eot-gibbs}
\end{align}
where the dual potentials $f, g$ are determined by the marginal
constraints (the normalization is absorbed into the potentials).

\paragraph{Discrete Measures.}
When $\mu = \sum_{i=1}^N a_i \delta_{x_i}$ and
$\nu = \sum_{j=1}^M b_j \delta_{y_j}$ with weights
$a \in \Delta^N$, $b \in \Delta^M$, the coupling is an $N$-by-$M$ matrix
$\Gamma$ and \eqref{eq:eot} reduces to
\begin{align}
    \T_c^\varepsilon(a, b) = \inf_{\Gamma \in \C[a, b]} \;
    \langle C, \Gamma \rangle
    \;+\; \varepsilon\, \KL{\Gamma}{a \otimes b},\label{eq:eot-d}
\end{align}
where, with some notational abuse, $C_{ij} = c(x_i, y_j)$ is the cost matrix,
$(a \otimes b)_{ij} = a_i b_j$, and the coupling polytope is
$\C[a, b] = \{\Gamma \geq 0 \mid \Gamma \mathbf{1}_M = a,\;
\Gamma^\top \mathbf{1}_N = b\}$.
The optimum has the form
$\Gamma^\star_{ij} = u_i K_{ij} v_j$, where
$K_{ij} = \exp(-C_{ij}/\varepsilon)$ is the \emph{Gibbs kernel} and
the scaling vectors $(u, v)$ are determined by the marginal
constraints. They are computed by the \emph{Sinkhorn algorithm}
(Algorithm~2), which alternately projects onto the
row and column marginal constraints and converges linearly with
$O(NM)$ cost per iteration~\citep{Cuturi13}.


\newpage
\section{Proofs of Theoretical Statements}
\label{app:proofs}
\chainrule* 
\begin{proof}
Expand the left-hand side using
$\dd(\mu \otimes \pi) = \dd(\mu \otimes \nu_\gamma) \cdot
\dd\nu_\gamma / \dd\pi$:
\begin{align*}
    \KL{\gamma}{\mu \otimes \pi}
    &= \int \log \frac{\dd\gamma}{\dd(\mu \otimes \pi)}\,
       \dd\gamma
    = \int \log \frac{\dd\gamma}{\dd(\mu \otimes \nu_\gamma)}\,
       \dd\gamma
    \;+\; \int \log \frac{\dd\nu_\gamma}{\dd\pi}(y)\,
       \dd\gamma(x, y) \\
    &= \KL{\gamma}{\mu \otimes \nu_\gamma}
    \;+\; \KL{\nu_\gamma}{\pi},
\end{align*}
where the last equality follows from,
\begin{align}
\int \log \frac{\dd\nu_\gamma}{\dd\pi}(y)\, \dd\gamma(x,y)
= \int \log \frac{\dd\nu_\gamma}{\dd\pi}(y)\, \dd\nu_\gamma(y),
\end{align}
by marginalization.
\end{proof}

\etd*
\begin{proof}
Inequality~\eqref{eq:upperbound} follows from
\Cref{lem:chainrule} and $I(\gamma) \geq 0$.
Equality holds when $I(\gamma) = 0$, i.e., $\gamma = \mu
\otimes \nu_\gamma$. Claim (i) follows by evaluating the gap
at the respective optima. For (ii), as $\varepsilon \to \infty$
the KL term dominates the transport cost in~\eqref{eq:etd},
driving $\gamma^\star$ toward the minimizer of
$\KL{\gamma}{\mu \otimes \pi}$ over $\C(\mu)$, which is the
product measure $\mu \otimes \pi$ (achieving $I = 0$). Claim
(iii) follows from the structure of~\eqref{eq:etd}: only the
source marginal is constrained ($\gamma \in \C(\mu)$), and
$\pi$ appears solely in the reference measure $\mu \otimes \pi$.
\end{proof}

\coupling*

\begin{proof}[Proof of Proposition~\ref{prop:coupling}]
The Lagrangian for~\eqref{eq:eot-tau} with row-marginal
constraint enforced by multipliers $f_i$ is
\begin{align*}
    L = \langle C, \Gamma \rangle
    + \varepsilon \sum_{ij} \Gamma_{ij}
      \log \frac{\Gamma_{ij}}{a_i b_j}
    + \tau \sum_j \nu_j \log \frac{\nu_j}{b_j}
    - \sum_i f_i \Bigl(\sum_j \Gamma_{ij} - a_i\Bigr),
\end{align*}
where $\nu_j = \sum_i \Gamma_{ij}$.  Setting
$\partial L / \partial \Gamma_{ij} = 0$ and solving for
$\Gamma_{ij}$ gives the Gibbs form
$\Gamma_{ij} = u_i\, K_{ij}\, v_j$,
where $K_{ij} = \exp(-C_{ij}/\varepsilon)$,
$u_i$ absorbs the multiplier $f_i$ and the $a_i$
dependence, and
\begin{align}
    v_j \;\propto\; b_j \cdot
    \bigl(\nu_j / b_j\bigr)^{-\rho},
    \qquad \rho = \tau / \varepsilon.
    \label{eq:v-implicit}
\end{align}
The row-marginal constraint $\sum_j \Gamma_{ij} = a_i$
determines $u_i = a_i / (Kv)_i$.  Substituting
$\nu_j = \sum_i \Gamma_{ij} = (K^\top u)_j\, v_j$
into~\eqref{eq:v-implicit} and collecting powers of $v_j$:
\begin{align}
    v_j = b_j \,\big/\,
    (K^\top u)_j^{\,\rho/(1+\rho)}.
    \label{eq:v-general}
\end{align}
The three regimes follow by specializing $\rho$:

\textbf{(i)} \emph{Semi-relaxed} ($\tau = 0$, $\rho = 0$).
The exponent vanishes, so $v_j = b_j$ independent of $u$.
No iteration is needed.

\textbf{(ii)} \emph{Unbalanced} ($0 < \tau < \infty$,
$0 < \rho < \infty$).  Equation~\eqref{eq:v-general}
couples $v$ to $u$ through $K^\top u$.  The pair $(u, v)$
is determined by the fixed-point iteration
$u_i \leftarrow a_i / (Kv)_i$,\; $v_j \leftarrow b_j / (K^\top u)_j^{\rho/(1+\rho)}$,
which converges linearly.

\textbf{(iii)} \emph{Balanced} ($\tau \to \infty$,
$\rho \to \infty$).  The exponent
$\rho/(1+\rho) \to 1$, giving
$v_j = b_j / (K^\top u)_j$---the standard Sinkhorn
iteration, which enforces both marginals exactly.
\end{proof}

\stationary*

\begin{proof}
We establish the three claims---characterization, independence, and
existence---in turn.
 
\paragraph{Stationarity characterization.}
In the population limit ($M \to \infty$), the pooled proposals from
particles distributed as~$\mu$ have aggregate density
\begin{equation}\label{eq:pooled-proposal}
  q_\mu(y)
  \;=\; \int q(y \mid x)\, \mu(x)\, dx
  \;=\; \int \mathcal{N}(y;\, x,\, \sigma^2 I)\, \mu(x)\, dx
  \;=\; [\mu * G_{\sigma^2}](y),
\end{equation}
where $q(y \mid x) = \mathcal{N}(y;\, x,\, \sigma^2 I)$ is the
score-free proposal kernel.
 
With standard target weights $b(y) \propto \pi(y)$, the effective target
marginal of the coupling is
\begin{equation}\label{eq:effective-target}
  \beta_\mu(y) \;\propto\; \pi(y)\, q_\mu(y)
  \;=\; \pi(y)\, [\mu * G_{\sigma^2}](y).
\end{equation}
This is the product of two sources of information: the target density
$\pi$ enters through the weights, and the proposal density $q_\mu$
enters through the support of the pooled proposals.
 
The balanced coupling $\gamma \in \mathcal{C}[\mu, \beta_\mu]$ enforces
both marginals exactly.  In particular, the target marginal constraint
gives, for any Borel set~$A$,
\begin{equation}\label{eq:target-marginal-constraint}
  \int \gamma(A \mid x)\, \mu(dx) \;=\; \beta_\mu(A).
\end{equation}
Under the categorical update, each particle $x_i$ is replaced by a
sample from $\gamma(\cdot \mid x_i)$.  In the mean-field limit
($N \to \infty$), the law of the updated ensemble is therefore
\[
  \mu^{\mathrm{new}}(A)
  \;=\; \int \gamma(A \mid x)\, \mu(dx)
  \;=\; \beta_\mu(A).
\]
That is, the aggregate output law after one ETD step is exactly
$\beta_\mu$, regardless of how individual particles are assigned.
 
The stationarity condition in \Cref{thm:stationary} requires $\mu^* = \beta_{\mu^*}$, i.e.:
\[
  \mu^*(y) \;=\; \beta_{\mu^*}(y)
  \;\propto\; \pi(y)\, [\mu^* * G_{\sigma^2}](y).
\]
Conversely, any $\mu^*$ satisfying this condition yields $\beta_{\mu^*} = \mu^*$, so the
dynamics are stationary.
 
\textbf{Independence from $c$ and $\varepsilon$.}
The key observation is that the balanced marginal constraint
\eqref{eq:target-marginal-constraint} is \emph{definitional}: the
Sinkhorn scaling vectors $u$ and $v$ adjust to enforce
$\Gamma^\top \mathbf{1}_N = \mathbf{b}$ regardless of the cost matrix
$C_{ij} = c(x_i, y_j)$ and the regularization parameter~$\varepsilon$.
Different choices of $(c, \varepsilon)$ change the internal structure of
the coupling---which specific proposal each particle is assigned
to---but the aggregate target marginal $\beta_\mu$ is invariant.
Since the stationarity condition $\mu^* = \beta_{\mu^*}$ involves only
$\beta_{\mu^*}$, the parameters $c$ and $\varepsilon$ do not appear in the stationarity condition.
 
\emph{Remark.}  While $(c, \varepsilon)$ do not affect the stationary
distribution, they do affect the transition kernel
$\gamma(\cdot \mid x)$ and therefore the \emph{mixing rate} of the
finite-$N$ chain.
 
\textbf{Existence.}
Define the nonlinear map
$T \colon \mathcal{P}_2(\mathbb{R}^n) \to \mathcal{P}_2(\mathbb{R}^n)$
by
\begin{equation}\label{eq:T-map}
  T(\mu)(y) \;=\;
  \frac{\pi(y)\, [\mu * G_{\sigma^2}](y)}
       {\int \pi(z)\, [\mu * G_{\sigma^2}](z)\, dz}.
\end{equation}
The denominator is finite for any $\mu \in \mathcal{P}_2(\mathbb{R}^n)$
and $\sigma > 0$, since $\pi$ has finite second moment and $G_{\sigma^2}$
is integrable.  A fixed point of $T$ is a solution to the stationarity condition.
 
To apply Schauder's fixed-point theorem, we work in the space
$\mathcal{P}_2(\mathbb{R}^n)$ equipped with the topology of weak
convergence.  We claim that:
\begin{enumerate}[label=(\alph*)]
  \item $T$ maps $\mathcal{P}_2(\mathbb{R}^n)$ into a tight (hence
        relatively compact) family.  To see this, note that
        $T(\mu)(y) \leq C_\pi\, \pi(y)$ for all $\mu$, where $C_\pi$
        depends only on $\pi$ and $\sigma$.  This follows because
        $[\mu * G_{\sigma^2}](y) \leq (2\pi\sigma^2)^{-n/2}$ uniformly
        in $y$ and $\mu$.  Since $\pi$ has a finite second moment, the
        family $\{T(\mu) : \mu \in \mathcal{P}_2(\mathbb{R}^n)\}$ is
        dominated by a fixed integrable function, which implies
        tightness.
  \item $T$ is continuous in the weak topology.  Let
        $\mu_k \rightharpoonup \mu$.  For each $y$,
        $[\mu_k * G_{\sigma^2}](y) = \int \mathcal{N}(y; x, \sigma^2 I)\,
        \mu_k(dx) \to [\mu * G_{\sigma^2}](y)$ by definition of weak
        convergence (since $x \mapsto \mathcal{N}(y; x, \sigma^2 I)$ is
        bounded and continuous).  Dominated convergence then gives
        $T(\mu_k) \rightharpoonup T(\mu)$.
\end{enumerate}
By Prokhorov's theorem, the image of $T$ is contained in a weakly
compact set $\mathcal{K} \subset \mathcal{P}_2(\mathbb{R}^n)$.  Since
$T(\mathcal{K}) \subseteq \mathcal{K}$ and $T$ is continuous, Schauder's
theorem guarantees a fixed point $\mu^* = T(\mu^*)$.
\end{proof}

\debiasing*

\begin{proof}
 
\textbf{(i) Stationarity.}
With importance-corrected weights, the effective target marginal
(cf.\ \eqref{eq:effective-target}) becomes
\begin{equation}\label{eq:ic-cancellation}
  \beta_\mu(y)
  \;\propto\; b^{\mathrm{IC}}(y)\, q_\mu(y)
  \;=\; \frac{\pi(y)}{q_\mu(y)} \cdot q_\mu(y)
  \;=\; \pi(y).
\end{equation}
The proposal density cancels exactly, leaving $\beta_\mu = \pi$
regardless of the current particle distribution~$\mu$.
 
Now assume $\mu = \pi$ (the stationarity hypothesis).  The balanced
coupling solves entropic OT between source~$\alpha = \pi$ and
target~$\beta = \pi$.  In the population limit, the Sinkhorn
coupling $\gamma$ has the form
\[
  \gamma(x, y) \;=\; u(x)\, K(x, y)\, v(y),
  \qquad K(x, y) = \exp\!\bigl(-c(x,y)/\varepsilon\bigr),
\]
where the scaling functions $u, v > 0$ enforce the marginal constraints:
\begin{equation}\label{eq:balanced-marginals-proof}
  \int \gamma(x, y)\, dy \;=\; \pi(x), \qquad
  \int \gamma(x, y)\, dx \;=\; \pi(y).
\end{equation}
Under the categorical update, the law of the updated ensemble is
\[
  \mu^{\mathrm{new}}(A)
  \;=\; \int \gamma(A \mid x)\, \pi(x)\, dx
  \;=\; \int_A \!\!\int \gamma(x,y)\, dx\, dy
  \;=\; \int_A \pi(y)\, dy
  \;=\; \pi(A),
\]
where the penultimate equality uses the target marginal constraint in
\eqref{eq:balanced-marginals-proof}.  Hence $\pi$ is stationary.
 
Note that this argument depends only on the marginal constraint, not on
the specific form of the Gibbs kernel or the value of~$\varepsilon$.
The stationarity therefore holds for any non-negative cost~$c$ and any
$\varepsilon > 0$.

\textbf{(ii) Reversibility (symmetric cost).} Assume $c(x,y) = c(y,x)$, so $K(x,y) = K(y,x)$.  Under the stationarity hypothesis $\mu = \pi$, the coupling solves balanced entropic OT between
$\pi$ (source) and $\pi$ (target) with symmetric kernel~$K$.
 
We show that the scaling functions coincide, $u = v$, which implies $\gamma(x,y) = \gamma(y,x)$. By the Sinkhorn factorization, $\gamma(x,y) = u(x)\, K(x,y)\, v(y)$
with $u, v$ satisfying
\begin{align}
  u(x) \int K(x,y)\, v(y)\, dy &= \pi(x), \label{eq:u-eq} \\
  v(y) \int K(x,y)\, u(x)\, dx &= \pi(y). \label{eq:v-eq}
\end{align}
Since $K$ is symmetric and both marginals equal~$\pi$, if $(u, v)$
satisfies \eqref{eq:u-eq}--\eqref{eq:v-eq}, then so does $(v, u)$
(swap the labels and use $K(x,y) = K(y,x)$).  By the uniqueness of the
Sinkhorn solution for strictly positive
kernels~\citep{PeyreCuturi19}, $u = v$ (up to a multiplicative constant
absorbed by the normalization).
 
Therefore,
\[
  \gamma(x, y) \;=\; u(x)\, K(x,y)\, u(y)
  \;=\; u(y)\, K(y,x)\, u(x)
  \;=\; \gamma(y, x).
\]
Detailed balance follows:
\[
  \pi(x)\, \gamma(y \mid x)
  \;=\; \gamma(x, y)
  \;=\; \gamma(y, x)
  \;=\; \pi(y)\, \gamma(x \mid y),
\]
so the chain is $\pi$-reversible.
\end{proof}

\subsection*{Metropolis Correction for the Semi-Relaxed Coupling}
\label{app:metropolis}

For the semi-relaxed coupling ($\tau = 0$) with score-free proposals 
($\alpha = 0$), the $M \to \infty$ transition kernel from $x$ to $y$ is
\begin{equation}\label{eq:sr-transition}
  K_{\mathrm{SR}}(x, \dd y) 
  \;=\; \frac{\pi(y)\,
    \exp\!\bigl(-\|x - y\|^2 / 2\varepsilon_\mathrm{eff}\bigr)}{Z(x)}\,\dd y,
\end{equation}
where $\varepsilon_\mathrm{eff}^{-1} = \varepsilon^{-1} + \sigma^{-2}$ 
combines the Gibbs kernel and the proposal variance, and:
\begin{align}
Z(x) = \int \pi(y)\,\exp(-\|x - y\|^2/2\varepsilon_\mathrm{eff})\,\dd y.
\end{align}

The Metropolis--Hastings acceptance ratio for a proposed move $x \to y$ is
\begin{align}
  \alpha_{\mathrm{MH}}(x, y)
  &= \min\!\biggl(1,\;
    \frac{\pi(y)\, K_{\mathrm{SR}}(y, x)}
         {\pi(x)\, K_{\mathrm{SR}}(x, y)}\biggr) \notag \\
  &= \min\!\biggl(1,\;
    \frac{\pi(y)}{\pi(x)} \cdot
    \frac{\pi(x)\,\exp(-\|y - x\|^2/2\varepsilon_\mathrm{eff})\,/\,Z(y)}
         {\pi(y)\,\exp(-\|x - y\|^2/2\varepsilon_\mathrm{eff})\,/\,Z(x)}
    \biggr) \notag \\
  &= \min\!\biggl(1,\; \frac{Z(x)}{Z(y)}\biggr).
  \label{eq:mh-ratio}
\end{align}
The simplification follows because the Gibbs kernel is symmetric in $(x,y)$
and the target density ratios cancel.  The corrected chain has stationary 
distribution~$\pi$ by construction.  Moves from low-density neighborhoods 
($Z(x)$ small) to high-density neighborhoods ($Z(y)$ large) are penalized, 
counteracting the over-concentration bias of 
\Cref{thm:stationary}.

\paragraph{Cost.}
The normalizing constant $Z(x_i) = \sum_j \pi(y_j)\,
\exp(-\|x_i - y_j\|^2/2\varepsilon_\mathrm{eff})$ is already computed
during the coupling step.  Evaluating $Z(y_j)$ at the accepted proposal 
requires the proposal--proposal kernel sums, adding $O(M^2)$ work per 
iteration---the same order as the existing $O(NM)$ coupling when $M > N$.

\paragraph{Score-guided case.}
When $\alpha > 0$, the proposal density 
$q_x(y) = \mathcal{N}(y;\, x + \alpha\,\partial_x \log\pi(x),\, \sigma^2 I)$
is asymmetric ($q_x(y) \neq q_y(x)$), introducing a MALA-type correction:
\begin{equation}\label{eq:mh-score}
  \alpha_{\mathrm{MH}}(x, y) 
  = \min\!\biggl(1,\; \frac{Z(x)}{Z(y)} \cdot
    \frac{q_y(x)}{q_x(y)}\biggr),
\end{equation}
which requires one additional score evaluation at the proposed point~$y$.

\begin{remark}[Metropolis Correction and the Balanced Coupling]
\label{rem:mh-balanced}
The Metropolis correction above applies to the \emph{semi-relaxed} coupling, 
where the conditional $\gamma(y \mid x_i)$ depends only on the current 
particle~$x_i$ and defines a single-particle transition kernel amenable to 
standard MH theory.  The \emph{balanced} coupling breaks this structure: the 
Sinkhorn scalings $(u, v)$ are determined jointly by both marginal constraints, 
making the conditional 
$\gamma(y \mid x_i) \propto u(x_i)\,K(x_i, y)\,v(y)$ 
a function of the \emph{entire} ensemble $\{x_1, \ldots, x_N\}$ through the 
dual variable~$v$.  This is the conditional of an interacting particle system, 
not a single-particle Markov kernel.  Standard MH requires a proposal kernel 
$Q(x, \dd y)$ that depends on~$x$ alone; a joint MH step over all~$N$ 
particles is formally valid but impractical, as the acceptance rate deteriorates 
exponentially in~$N$.  The balanced coupling instead admits the 
importance-corrected weights of \Cref{thm:debiasing} as its bias-elimination 
mechanism.
\end{remark}

\begin{table}[h]
\centering
\caption{Bias-correction strategies across the $\tau$-family.  ``N/A'' 
indicates the correction is structurally ill-defined 
(\Cref{rem:mh-balanced}).}
\label{tab:bias-correction}
\small
\begin{tabular}{@{}lcc@{}}
\toprule
\textbf{Correction} & \textbf{Semi-relaxed} ($\tau = 0$) 
  & \textbf{Balanced} ($\tau \to \infty$) \\
\midrule
None ($b \propto \pi(y)$) 
  & Biased & Biased \\
Importance ($b \propto \pi/q$) 
  & Bias reduced & \textbf{Unbiased} (\Cref{thm:debiasing}) \\
Metropolis--Hastings 
  & \textbf{Unbiased} \eqref{eq:mh-ratio} & N/A \\
\bottomrule
\end{tabular}
\end{table}



\newpage
\section{Experiment Setup}
\label{app:experiment_setup}

All ETD method names follow the convention \texttt{ETD-\{coupling\}-\{cost\}}, with optional score-guided or score-free mode indicated in the surrounding table section. The coupling label specifies the entropic transport constraint family: \texttt{SR} denotes the semi-relaxed coupling, \texttt{UB} denotes the unbalanced coupling, and \texttt{BAL} denotes the balanced coupling. The cost label specifies the transport cost: \texttt{Euc} uses squared Euclidean distance, \texttt{Maha} uses a diagonal Mahalanobis distance based on the current particle covariance, and \texttt{Mom} uses the Euclidean cost with coupling-momentum proposal drift. An \texttt{(IS)} suffix indicates proposal-density importance correction is applied when computing the Gibbs weights used as the target marginal in the transport coupling. Hyperparameter symbols used in the tables are summarized in Tables \ref{tab-app:etd_hyperparam_key} and \ref{tab-app:baseline_hyperparam_key}. The Gibbs inverse temperature $\beta$ rescales the target weights as $b_j \propto \pi(y_j)^\beta$, which is equivalent to replacing the potential $V$ with $\beta V$; this does not change the algorithm structurally but allows the optimizer to control the sharpness of the target weighting independently of the proposal parameters.

We ran all the experiments on an Intel i9-13900K PC with 64GB RAM and an NVIDIA RTX 4090 (24GB VRAM) GPU. All methods are implemented in Python using the JAX library for vectorization and efficient GPU execution. Runtime varied by benchmark from minutes for the 2D energy-function and variance-collapse experiments to several hours for the larger Covertype, BNN UCI, and LJ-13 experiments; all reported runs fit within the 24GB memory of a single RTX 4090.

\begin{table}[h]
  \caption{Hyperparameter symbols used in ETD tuning tables.}
  \centering
  \small
  \begin{tabular}{@{}ll@{}}
    \toprule
    \textbf{Symbol} & \textbf{Meaning} \\
    \midrule
    $\epsilon$ & Entropic regularization in the transport coupling \\
    $\beta$    & Gibbs inverse temperature used to weight proposals by target density \\
    $\alpha$   & Fraction of proposals drawn from local particle-centered proposals \\
    $\sigma$   & Standard deviation of local Gaussian proposal noise \\
    $\tau$     & Unbalanced marginal-relaxation parameter; used only by ETD-UB \\
    $\mu$      & Coupling-momentum coefficient; used only by Mom variants \\
    \bottomrule
  \end{tabular}
  \label{tab-app:etd_hyperparam_key}
\end{table}

\begin{table}[h]
  \caption{Baseline hyperparameter notation used in appendix tables.}
  \centering
  \small
  \begin{tabular}{@{}lll@{}}
    \toprule
    \textbf{Method} & \textbf{Symbol} & \textbf{Meaning} \\
    \midrule
    SVGD     & step size        & AdaGrad/SVGD learning-rate scale \\
    SGLD     & $a$              & numerator in the decaying step-size schedule \\
    SGLD     & $b$              & additive offset in the decaying step-size schedule \\
    SGLD     & $\gamma$         & decay exponent in the step-size schedule \\
    AGF-SVGD & lr               & learning rate \\
    AGF-SVGD & annealing stages & number of annealing stages \\
    AGF-SVGD & smoothing bw     & kernel-density smoothing bandwidth scale \\
    AGF-SVGD & base scale       & base proposal/kernel scale used by AGF-SVGD \\
    \bottomrule
  \end{tabular}
  \label{tab-app:baseline_hyperparam_key}
\end{table}

\subsection{Variance Collapse}

\textbf{Tuning.} All methods are tuned at $d=50$ with $|\mathrm{DAMV}-1|$ as the tuning objective, averaged over 3 seeds per trial. All methods use 50 Optuna TPE trials. Tuning uses 500 iterations while final evaluation uses 2000 iterations across $d \in \{10, 20, 50, 100, 200\}$ with 5 seeds. The hyperparameters are provided in Table~\ref{tab-app:vc_tuned}.

\begin{table}[h]
  \caption{Variance Collapse: Tuned hyperparameters for ETD variants shown in Table~\ref{tab-app:vc_results}. All use $N=50$ particles and 2000 iterations (tuned at $d=50$ with 500 iterations). For ETD, $M=500$ proposals.}
  \centering
  \scriptsize
  \begin{tabular}{@{}l cccccc@{}}
    \toprule
    \textbf{Method} & $\epsilon$ & $\beta$ & $\alpha$ & $\sigma$ & $\tau$ & $\mu$ \\
    \midrule
    \multicolumn{7}{@{}l}{\textit{Score-free}} \\[2pt]
    ETD-BAL-Euc       & 0.404 & 0.221 & 0.929 & 0.579 & ---   & ---   \\
    ETD-BAL-Maha      & 0.676 & 0.185 & 0.539 & 0.541 & ---   & ---   \\
    ETD-BAL-Mom       & 0.190 & 0.168 & 0.710 & 0.583 & ---   & 0.352 \\
    ETD-SR-Euc        & 2.428 & 0.381 & 0.756 & 0.858 & ---   & ---   \\
    ETD-SR-Maha       & 1.445 & 0.281 & 0.863 & 1.000 & ---   & ---   \\
    ETD-SR-Mom        & 0.519 & 0.175 & 0.840 & 0.724 & ---   & 0.397 \\
    ETD-UB-Maha       & 1.225 & 0.336 & 0.962 & 0.733 & 8.810 & ---   \\
    ETD-UB-Maha (IS)  & 4.807 & 1.129 & 0.818 & 0.654 & 1.975 & ---   \\
    ETD-UB-Mom        & 0.418 & 0.132 & 0.624 & 0.501 & 4.447 & 0.324 \\
    \midrule
    \multicolumn{7}{@{}l}{\textit{Score-guided}} \\[2pt]
    ETD-BAL-Euc       & 0.062 & 0.283 & 0.863 & 0.861 & ---   & ---   \\
    ETD-BAL-Euc (IS)  & 1.448 & 0.971 & 0.993 & 1.144 & ---   & ---   \\
    ETD-BAL-Maha      & 0.058 & 0.190 & 0.642 & 0.749 & ---   & ---   \\
    ETD-BAL-Mom (IS)  & 1.785 & 0.927 & 0.520 & 1.082 & ---   & 0.123 \\
    ETD-SR-Euc        & 0.071 & 0.246 & 0.982 & 0.346 & ---   & ---   \\
    ETD-SR-Euc (IS)   & 0.050 & 1.565 & 0.817 & 0.426 & ---   & ---   \\
    ETD-SR-Maha (IS)  & 0.093 & 0.541 & 0.997 & 0.627 & ---   & ---   \\
    ETD-SR-Mom        & 0.068 & 0.174 & 0.781 & 0.326 & ---   & 0.079 \\
    ETD-UB-Maha       & 0.163 & 0.313 & 0.212 & 0.664 & 0.268 & ---   \\
    ETD-UB-Mom        & 0.473 & 0.200 & 0.896 & 0.822 & 4.307 & 0.405 \\
    \midrule
    \multicolumn{7}{@{}l}{\textit{Baselines}} \\[2pt]
    SVGD     & \multicolumn{6}{l}{step size $= 0.0817$} \\
    SGLD     & \multicolumn{6}{l}{$a = 0.0883$,\; $b = 1.0$,\; $\gamma = 0.55$} \\
    AGF-SVGD & \multicolumn{6}{l}{lr $= 1.71 \times 10^{-4}$,\; annealing stages $= 59$,\; smoothing bw $= 0.077$,\; base scale $= 1.37$} \\
    \bottomrule
  \end{tabular}
  \label{tab-app:vc_tuned}
\end{table}

\begin{table}[h]
  \caption{Variance Collapse: Results for the top 3 score-free and top 3 score-guided ETD variants per dimension. DAMV: dimension-averaged marginal variance (mean $\pm$ SE, 5 seeds); $|\text{DAMV}-1|$: absolute deviation from target variance ($\text{DAMV}=1$ is exact recovery).}
  \centering
  \scriptsize
  \setlength{\tabcolsep}{2.5pt}

  \begin{minipage}[t]{0.48\textwidth}\centering
  \textbf{$d=10$}\\[3pt]
  \begin{tabular}{@{}l cc@{}}
    \toprule
    \textbf{Method} & DAMV & $|\text{DAMV}{-}1|$ \\
    \midrule
    \multicolumn{3}{@{}l}{\textit{Score-free}} \\[1pt]
    ETD-BAL-Euc       & $1.026 \pm 0.029$ & $0.026$ \\
    ETD-SR-Euc        & $1.037 \pm 0.026$ & $0.037$ \\
    ETD-BAL-Mom       & $0.943 \pm 0.038$ & $0.057$ \\
    \multicolumn{3}{@{}l}{\textit{Score-guided}} \\[1pt]
    ETD-BAL-Euc (IS)  & $1.005 \pm 0.027$ & $0.005$ \\
    ETD-UB-Maha       & $0.988 \pm 0.045$ & $0.012$ \\
    ETD-UB-Mom        & $1.033 \pm 0.061$ & $0.033$ \\
    \multicolumn{3}{@{}l}{\textit{Baselines}} \\[1pt]
    SVGD              & $0.415 \pm 0.008$ & $0.585$ \\
    SGLD              & $0.980 \pm 0.026$ & $0.020$ \\
    AGF-SVGD          & $4.058 \pm 0.145$ & $3.058$ \\
    \bottomrule
  \end{tabular}
  \end{minipage}%
  \hfill
  \begin{minipage}[t]{0.48\textwidth}\centering
  \textbf{$d=20$}\\[3pt]
  \begin{tabular}{@{}l cc@{}}
    \toprule
    \textbf{Method} & DAMV & $|\text{DAMV}{-}1|$ \\
    \midrule
    \multicolumn{3}{@{}l}{\textit{Score-free}} \\[1pt]
    ETD-UB-Maha       & $1.007 \pm 0.049$ & $0.007$ \\
    ETD-BAL-Maha      & $1.011 \pm 0.053$ & $0.011$ \\
    ETD-BAL-Mom       & $1.034 \pm 0.020$ & $0.034$ \\
    \multicolumn{3}{@{}l}{\textit{Score-guided}} \\[1pt]
    ETD-BAL-Euc (IS)  & $1.001 \pm 0.026$ & $0.001$ \\
    ETD-UB-Maha       & $1.007 \pm 0.039$ & $0.007$ \\
    ETD-SR-Maha (IS)  & $0.957 \pm 0.058$ & $0.043$ \\
    \multicolumn{3}{@{}l}{\textit{Baselines}} \\[1pt]
    SVGD              & $0.268 \pm 0.005$ & $0.732$ \\
    SGLD              & $0.988 \pm 0.021$ & $0.012$ \\
    AGF-SVGD          & $4.048 \pm 0.085$ & $3.048$ \\
    \bottomrule
  \end{tabular}
  \end{minipage}

  \vspace{10pt}

  \begin{minipage}[t]{0.48\textwidth}\centering
  \textbf{$d=50$ (tuning dimension)}\\[3pt]
  \begin{tabular}{@{}l cc@{}}
    \toprule
    \textbf{Method} & DAMV & $|\text{DAMV}{-}1|$ \\
    \midrule
    \multicolumn{3}{@{}l}{\textit{Score-free}} \\[1pt]
    ETD-UB-Maha       & $0.964 \pm 0.045$ & $0.036$ \\
    ETD-BAL-Mom       & $1.036 \pm 0.024$ & $0.036$ \\
    ETD-UB-Maha (IS)  & $0.963 \pm 0.092$ & $0.037$ \\
    \multicolumn{3}{@{}l}{\textit{Score-guided}} \\[1pt]
    ETD-BAL-Euc (IS)  & $1.002 \pm 0.016$ & $0.002$ \\
    ETD-BAL-Euc       & $0.998 \pm 0.054$ & $0.002$ \\
    ETD-SR-Euc        & $1.004 \pm 0.010$ & $0.004$ \\
    \multicolumn{3}{@{}l}{\textit{Baselines}} \\[1pt]
    SVGD              & $0.159 \pm 0.004$ & $0.841$ \\
    SGLD              & $0.980 \pm 0.014$ & $0.020$ \\
    AGF-SVGD          & $3.919 \pm 0.045$ & $2.919$ \\
    \bottomrule
  \end{tabular}
  \end{minipage}%
  \hfill
  \begin{minipage}[t]{0.48\textwidth}\centering
  \textbf{$d=100$}\\[3pt]
  \begin{tabular}{@{}l cc@{}}
    \toprule
    \textbf{Method} & DAMV & $|\text{DAMV}{-}1|$ \\
    \midrule
    \multicolumn{3}{@{}l}{\textit{Score-free}} \\[1pt]
    ETD-SR-Mom        & $1.002 \pm 0.023$ & $0.002$ \\
    ETD-BAL-Mom       & $1.029 \pm 0.033$ & $0.029$ \\
    ETD-UB-Mom        & $0.769 \pm 0.016$ & $0.231$ \\
    \multicolumn{3}{@{}l}{\textit{Score-guided}} \\[1pt]
    ETD-BAL-Maha      & $0.991 \pm 0.066$ & $0.010$ \\
    ETD-SR-Euc        & $0.988 \pm 0.010$ & $0.012$ \\
    ETD-BAL-Mom (IS)  & $1.021 \pm 0.007$ & $0.021$ \\
    \multicolumn{3}{@{}l}{\textit{Baselines}} \\[1pt]
    SVGD              & $0.127 \pm 0.002$ & $0.873$ \\
    SGLD              & $0.975 \pm 0.012$ & $0.025$ \\
    AGF-SVGD          & $3.958 \pm 0.048$ & $2.958$ \\
    \bottomrule
  \end{tabular}
  \end{minipage}

  \vspace{10pt}

  \begin{minipage}[t]{0.48\textwidth}\centering
  \textbf{$d=200$}\\[3pt]
  \begin{tabular}{@{}l cc@{}}
    \toprule
    \textbf{Method} & DAMV & $|\text{DAMV}{-}1|$ \\
    \midrule
    \multicolumn{3}{@{}l}{\textit{Score-free}} \\[1pt]
    ETD-SR-Mom        & $0.713 \pm 0.062$ & $0.287$ \\
    ETD-SR-Maha       & $0.670 \pm 0.292$ & $0.330$ \\
    ETD-UB-Maha (IS)  & $0.542 \pm 0.074$ & $0.458$ \\
    \multicolumn{3}{@{}l}{\textit{Score-guided}} \\[1pt]
    ETD-BAL-Mom (IS)  & $1.003 \pm 0.017$ & $0.003$ \\
    ETD-SR-Euc (IS)   & $0.971 \pm 0.041$ & $0.029$ \\
    ETD-SR-Mom        & $1.034 \pm 0.011$ & $0.034$ \\
    \multicolumn{3}{@{}l}{\textit{Baselines}} \\[1pt]
    SVGD              & $0.110 \pm 0.003$ & $0.890$ \\
    SGLD              & $0.974 \pm 0.006$ & $0.026$ \\
    AGF-SVGD          & $3.971 \pm 0.033$ & $2.971$ \\
    \bottomrule
  \end{tabular}
  \end{minipage}
  \label{tab-app:vc_results}
\end{table}

\subsection{2D Energy Functions}

\textbf{Description.} We use the four two-dimensional energy-function targets $U_1$--$U_4$ adapted from \citet{Rezende15}. For $U_2$--$U_4$, we add a weak confinement term $+0.1|z_1|$ to the energy so that the targets define proper normalizable distributions on $\mathbb{R}^2$, enabling exact reference sampling. We additionally include an 8-mode ring Gaussian mixture ($\sigma=0.5$, radius 5) as a mode-coverage stress target. Ground truth reference samples (10K per target) are generated via grid sampling for $U_1$ and exact sampling for $U_2$--$U_4$ and the ring GMM.

\textbf{Tuning.} Each method is tuned independently on each of the five targets with mean absolute energy-distance as the tuning objective, averaged over 3 seeds per trial. All methods use 50 Optuna TPE trials per target. The hyperparameters are provided in Table~\ref{tab-app:viz2d_tuned}.

\begin{table}[h]
  \caption{2D Energy Functions: Tuned hyperparameters for the top 3 score-free and top 3 score-guided ETD variants per target. All use $N=50$ particles and 500 iterations. For ETD, $M=500$ proposals.}
  \centering
  \scriptsize
  \begin{tabular}{@{}l cccccc@{}}
    \toprule
    \textbf{Method} & $\epsilon$ & $\beta$ & $\alpha$ & $\sigma$ & $\tau$ & $\mu$ \\
    \midrule
    \multicolumn{7}{@{}l}{\textbf{$U_1$ (Ring + bumps)}} \\[2pt]
    \multicolumn{7}{@{}l}{\textit{Score-free}} \\[2pt]
    ETD-BAL-Maha (IS) & 0.0176 & 0.276 & 0.355 & 1.23 & ---   & ---   \\
    ETD-UB-Mom (IS)   & 0.123  & 1.82  & 0.938 & 1.21 & 6.88  & 0.228 \\
    ETD-UB-Euc (IS)   & 0.072  & 0.322 & 0.639 & 1.30 & 0.158 & ---   \\
    \multicolumn{7}{@{}l}{\textit{Score-guided}} \\[2pt]
    ETD-UB-Maha       & 0.0693 & 0.172 & 0.242 & 0.135 & 9.51 & ---   \\
    ETD-BAL-Maha      & 0.0402 & 0.851 & 0.529 & 1.51  & ---  & ---   \\
    ETD-BAL-Euc       & 0.0818 & 0.478 & 0.305 & 1.54  & ---  & ---   \\
    \multicolumn{7}{@{}l}{\textit{Baselines}} \\[2pt]
    SVGD     & \multicolumn{6}{l}{step size $= 0.0558$} \\
    SGLD     & \multicolumn{6}{l}{$a = 0.717$,\; $b = 1.0$,\; $\gamma = 0.55$} \\
    AGF-SVGD & \multicolumn{6}{l}{lr $= 0.119$,\; annealing stages $= 191$,\; smoothing bw $= 0.522$,\; base scale $= 2.61$} \\
    \midrule
    \multicolumn{7}{@{}l}{\textbf{$U_2$ (Banana)}} \\[2pt]
    \multicolumn{7}{@{}l}{\textit{Score-free}} \\[2pt]
    ETD-UB-Mom (IS)   & 0.126  & 0.569 & 0.794 & 1.58 & 4.48  & 0.249 \\
    ETD-SR-Mom (IS)   & 0.122  & 0.463 & 0.807 & 2.79 & ---   & 0.075 \\
    ETD-SR-Euc (IS)   & 0.107  & 0.481 & 0.865 & 2.07 & ---   & ---   \\
    \multicolumn{7}{@{}l}{\textit{Score-guided}} \\[2pt]
    ETD-UB-Mom (IS)   & 0.112  & 0.404 & 0.426 & 1.89 & 0.421 & 0.428 \\
    ETD-BAL-Mom (IS)  & 0.0187 & 0.973 & 0.936 & 1.18 & ---   & 0.187 \\
    ETD-BAL-Maha (IS) & 0.0229 & 0.920 & 0.915 & 1.37 & ---   & ---   \\
    \multicolumn{7}{@{}l}{\textit{Baselines}} \\[2pt]
    SVGD     & \multicolumn{6}{l}{step size $= 0.903$} \\
    SGLD     & \multicolumn{6}{l}{$a = 0.369$,\; $b = 1.0$,\; $\gamma = 0.55$} \\
    AGF-SVGD & \multicolumn{6}{l}{lr $= 0.967$,\; annealing stages $= 872$,\; smoothing bw $= 0.436$,\; base scale $= 19.0$} \\
    \midrule
    \multicolumn{7}{@{}l}{\textbf{$U_3$ (Parallel modes)}} \\[2pt]
    \multicolumn{7}{@{}l}{\textit{Score-free}} \\[2pt]
    ETD-BAL-Euc (IS)  & 0.0123 & 0.543 & 0.553 & 1.46 & ---  & ---   \\
    ETD-BAL-Mom (IS)  & 0.0285 & 1.03  & 0.953 & 1.57 & ---  & 0.061 \\
    ETD-SR-Mom (IS)   & 0.0526 & 0.327 & 0.881 & 1.06 & ---  & 0.444 \\
    \multicolumn{7}{@{}l}{\textit{Score-guided}} \\[2pt]
    ETD-UB-Mom (IS)   & 0.348  & 0.708 & 0.955 & 1.36 & 3.63  & 0.232 \\
    ETD-UB-Euc (IS)   & 0.135  & 0.497 & 0.591 & 1.41 & 0.301 & ---   \\
    ETD-BAL-Maha (IS) & 0.0291 & 0.789 & 0.660 & 1.51 & ---   & ---   \\
    \multicolumn{7}{@{}l}{\textit{Baselines}} \\[2pt]
    SVGD     & \multicolumn{6}{l}{step size $= 0.998$} \\
    SGLD     & \multicolumn{6}{l}{$a = 0.304$,\; $b = 1.0$,\; $\gamma = 0.55$} \\
    AGF-SVGD & \multicolumn{6}{l}{lr $= 0.949$,\; annealing stages $= 752$,\; smoothing bw $= 0.353$,\; base scale $= 112.8$} \\
    \midrule
    \multicolumn{7}{@{}l}{\textbf{$U_4$ (Sigmoid offset)}} \\[2pt]
    \multicolumn{7}{@{}l}{\textit{Score-free}} \\[2pt]
    ETD-UB-Euc (IS)   & 0.0197 & 0.826 & 0.638 & 1.43  & 1.27 & ---   \\
    ETD-BAL-Mom (IS)  & 0.0152 & 0.902 & 0.900 & 1.11  & ---  & 0.190 \\
    ETD-BAL-Euc (IS)  & 0.123  & 0.750 & 0.788 & 2.89  & ---  & ---   \\
    \multicolumn{7}{@{}l}{\textit{Score-guided}} \\[2pt]
    ETD-SR-Mom (IS)   & 0.0467 & 0.265 & 0.629 & 0.408 & ---  & 0.350 \\
    ETD-BAL-Maha (IS) & 0.401  & 0.338 & 0.191 & 0.623 & ---  & ---   \\
    ETD-UB-Maha (IS)  & 0.0484 & 0.277 & 0.258 & 0.519 & 5.73 & ---   \\
    \multicolumn{7}{@{}l}{\textit{Baselines}} \\[2pt]
    SVGD     & \multicolumn{6}{l}{step size $= 0.903$} \\
    SGLD     & \multicolumn{6}{l}{$a = 0.268$,\; $b = 1.0$,\; $\gamma = 0.55$} \\
    AGF-SVGD & \multicolumn{6}{l}{lr $= 0.643$,\; annealing stages $= 965$,\; smoothing bw $= 0.056$,\; base scale $= 80.4$} \\
    \midrule
    \multicolumn{7}{@{}l}{\textbf{8-mode Ring GMM}} \\[2pt]
    \multicolumn{7}{@{}l}{\textit{Score-free}} \\[2pt]
    ETD-UB-Euc       & 0.125 & 0.270 & 0.629 & 1.27  & 2.19 & ---   \\
    ETD-BAL-Mom (IS) & 0.452 & 0.503 & 0.907 & 0.758 & ---  & 0.384 \\
    ETD-UB-Mom       & 0.223 & 0.228 & 0.106 & 4.33  & 14.8 & 0.483 \\
    \multicolumn{7}{@{}l}{\textit{Score-guided}} \\[2pt]
    ETD-BAL-Maha     & 0.152  & 0.197 & 0.327 & 2.15 & ---  & ---   \\
    ETD-BAL-Mom      & 1.24   & 0.230 & 0.458 & 2.37 & ---  & 0.061 \\
    ETD-SR-Euc       & 0.0103 & 2.56  & 0.446 & 1.05 & ---  & ---   \\
    \multicolumn{7}{@{}l}{\textit{Baselines}} \\[2pt]
    SVGD     & \multicolumn{6}{l}{step size $= 0.0329$} \\
    SGLD     & \multicolumn{6}{l}{$a = 3.94$,\; $b = 1.0$,\; $\gamma = 0.55$} \\
    AGF-SVGD & \multicolumn{6}{l}{lr $= 0.0811$,\; annealing stages $= 589$,\; smoothing bw $= 0.022$,\; base scale $= 3.71$} \\
    \bottomrule
  \end{tabular}
  \label{tab-app:viz2d_tuned}
\end{table}

\begin{table}[h]
  \caption{2D Energy Functions: Results for the top 3 score-free and top 3 score-guided ETD variants per target. $|E_d|$: absolute energy distance (mean $\pm$ SE, 5 seeds). MMD$^2$: squared maximum mean discrepancy (mean).}
  \centering
  \scriptsize
  \setlength{\tabcolsep}{2.5pt}

  \begin{minipage}[t]{0.48\textwidth}\centering
  \textbf{$U_1$ (Ring + bumps)}\\[3pt]
  \begin{tabular}{@{}l cc@{}}
    \toprule
    \textbf{Method} & $|E_d|$ & MMD$^2$ \\
    \midrule
    \multicolumn{3}{@{}l}{\textit{Score-free}} \\[1pt]
    ETD-BAL-Maha (IS) & $0.002 \pm 0.005$ & $-0.002$ \\
    ETD-UB-Mom (IS)   & $0.002 \pm 0.006$ & $-0.002$ \\
    ETD-UB-Euc (IS)   & $0.004 \pm 0.004$ & $-0.003$ \\
    \multicolumn{3}{@{}l}{\textit{Score-guided}} \\[1pt]
    ETD-UB-Maha       & $0.001 \pm 0.010$ & $-0.002$ \\
    ETD-BAL-Maha      & $0.001 \pm 0.011$ & $-0.001$ \\
    ETD-BAL-Euc       & $0.004 \pm 0.007$ & $-0.001$ \\
    \multicolumn{3}{@{}l}{\textit{Baselines}} \\[1pt]
    SVGD     & $0.049 \pm 0.030$ & $0.011$ \\
    SGLD     & $0.017 \pm 0.009$ & $-0.003$ \\
    AGF-SVGD & $0.278 \pm 0.134$ & $0.048$ \\
    \bottomrule
  \end{tabular}
  \end{minipage}%
  \hfill
  \begin{minipage}[t]{0.48\textwidth}\centering
  \textbf{$U_2$ (Banana)}\\[3pt]
  \begin{tabular}{@{}l cr@{}}
    \toprule
    \textbf{Method} & $|E_d|$ & MMD$^2$ \\
    \midrule
    \multicolumn{3}{@{}l}{\textit{Score-free}} \\[1pt]
    ETD-UB-Mom (IS)   & $0.017 \pm 0.069$ & $-0.000$ \\
    ETD-SR-Mom (IS)   & $0.021 \pm 0.046$ & $0.004$ \\
    ETD-SR-Euc (IS)   & $0.097 \pm 0.066$ & $0.004$ \\
    \multicolumn{3}{@{}l}{\textit{Score-guided}} \\[1pt]
    ETD-UB-Mom (IS)   & $0.009 \pm 0.047$ & $-0.000$ \\
    ETD-BAL-Mom (IS)  & $0.058 \pm 0.049$ & $0.005$ \\
    ETD-BAL-Maha (IS) & $0.063 \pm 0.035$ & $0.001$ \\
    \multicolumn{3}{@{}l}{\textit{Baselines}} \\[1pt]
    SVGD     & $1.828 \pm 0.058$ & $0.099$ \\
    SGLD     & $0.039 \pm 0.078$ & $0.001$ \\
    AGF-SVGD & $0.322 \pm 0.118$ & $0.014$ \\
    \bottomrule
  \end{tabular}
  \end{minipage}

  \vspace{10pt}

  \begin{minipage}[t]{0.48\textwidth}\centering
  \textbf{$U_3$ (Parallel modes)}\\[3pt]
  \begin{tabular}{@{}l cc@{}}
    \toprule
    \textbf{Method} & $|E_d|$ & MMD$^2$ \\
    \midrule
    \multicolumn{3}{@{}l}{\textit{Score-free}} \\[1pt]
    ETD-BAL-Euc (IS)  & $0.014 \pm 0.094$ & $0.004$ \\
    ETD-BAL-Mom (IS)  & $0.023 \pm 0.044$ & $0.001$ \\
    ETD-SR-Mom (IS)   & $0.082 \pm 0.125$ & $0.001$ \\
    \multicolumn{3}{@{}l}{\textit{Score-guided}} \\[1pt]
    ETD-UB-Mom (IS)   & $0.060 \pm 0.025$ & $-0.003$ \\
    ETD-UB-Euc (IS)   & $0.063 \pm 0.055$ & $-0.002$ \\
    ETD-BAL-Maha (IS) & $0.143 \pm 0.040$ & $0.007$ \\
    \multicolumn{3}{@{}l}{\textit{Baselines}} \\[1pt]
    SVGD     & $2.002 \pm 0.045$ & $0.105$ \\
    SGLD     & $0.067 \pm 0.064$ & $0.002$ \\
    AGF-SVGD & $0.764 \pm 0.206$ & $0.032$ \\
    \bottomrule
  \end{tabular}
  \end{minipage}%
  \hfill
  \begin{minipage}[t]{0.48\textwidth}\centering
  \textbf{$U_4$ (Sigmoid offset)}\\[3pt]
  \begin{tabular}{@{}l cc@{}}
    \toprule
    \textbf{Method} & $|E_d|$ & MMD$^2$ \\
    \midrule
    \multicolumn{3}{@{}l}{\textit{Score-free}} \\[1pt]
    ETD-UB-Euc (IS)   & $0.061 \pm 0.027$ & $-0.002$ \\
    ETD-BAL-Mom (IS)  & $0.087 \pm 0.019$ & $0.004$ \\
    ETD-BAL-Euc (IS)  & $0.191 \pm 0.202$ & $0.008$ \\
    \multicolumn{3}{@{}l}{\textit{Score-guided}} \\[1pt]
    ETD-SR-Mom (IS)   & $0.540 \pm 0.082$ & $0.027$ \\
    ETD-BAL-Maha (IS) & $0.722 \pm 0.195$ & $0.025$ \\
    ETD-UB-Maha (IS)  & $1.228 \pm 0.329$ & $0.062$ \\
    \multicolumn{3}{@{}l}{\textit{Baselines}} \\[1pt]
    SVGD     & $1.832 \pm 0.054$ & $0.098$ \\
    SGLD     & $0.662 \pm 0.048$ & $0.035$ \\
    AGF-SVGD & $0.514 \pm 0.119$ & $0.015$ \\
    \bottomrule
  \end{tabular}
  \end{minipage}

  \vspace{10pt}

  \begin{minipage}[t]{0.48\textwidth}\centering
  \textbf{8-mode Ring GMM}\\[3pt]
  \begin{tabular}{@{}l cc@{}}
    \toprule
    \textbf{Method} & $|E_d|$ & MMD$^2$ \\
    \midrule
    \multicolumn{3}{@{}l}{\textit{Score-free}} \\[1pt]
    ETD-UB-Euc       & $0.001 \pm 0.034$ & $-0.001$ \\
    ETD-BAL-Mom (IS) & $0.003 \pm 0.049$ & $-0.001$ \\
    ETD-UB-Mom       & $0.005 \pm 0.016$ & $-0.002$ \\
    \multicolumn{3}{@{}l}{\textit{Score-guided}} \\[1pt]
    ETD-BAL-Maha     & $0.001 \pm 0.015$ & $-0.001$ \\
    ETD-BAL-Mom      & $0.001 \pm 0.019$ & $-0.002$ \\
    ETD-SR-Euc       & $0.003 \pm 0.018$ & $-0.001$ \\
    \multicolumn{3}{@{}l}{\textit{Baselines}} \\[1pt]
    SVGD     & $0.071 \pm 0.058$ & $0.001$ \\
    SGLD     & $0.050 \pm 0.013$ & $-0.004$ \\
    AGF-SVGD & $0.085 \pm 0.016$ & $0.006$ \\
    \bottomrule
  \end{tabular}
  \end{minipage}
  \label{tab-app:viz2d_results}
\end{table}

\subsection{Bayesian Logistic Regression}

For binary classification, the model uses a conditional Gaussian prior $\mathcal{N}(0, \alpha_b^{-1}I)$ on regression weight $w$ and a Gamma hyperprior on the precision $\alpha_b$. The inference target is the joint posterior $P(\theta |D)$ where, $\theta = [w, \log{\alpha_b}]$.

\subsubsection{German Credit}

All methods are tuned on a single 90/10 train/test split of German Credit with NLL as the tuning objective. ETD variants and AGF-SVGD are tuned with 200 Optuna TPE trials each while SVGD and SGLD use 50 trials. The hyperparameters are provided in Table \ref{tab-app:blr_german_tuned}.

\begin{table}[h]
  \caption{BLR German Credit: Tuned hyperparameters for selected score-free
  and score-guided ETD variants. All use $N=100$ particles and 2000 iterations.
  For ETD, $M=500$ proposals.}
  \centering
  \small
  \begin{tabular}{@{}l cccccc@{}}
    \toprule
    \textbf{Method} & $\epsilon$ & $\beta$ & $\alpha$ & $\sigma$ & $\tau$ & $\mu$ \\
    \midrule
    \multicolumn{7}{@{}l}{\textit{Score-free}} \\[2pt]
    ETD-UB-Mom (IS) & 0.132   & 103.3 & 0.370 & 0.0014 & 0.031 & 0.339 \\
    ETD-UB-Euc      & 0.166   & 0.150 & 0.077 & 0.0043 & 0.309 & ---   \\
    ETD-SR-Mom      & 0.037   & 0.037 & 0.340 & 0.0016 & ---   & 0.423 \\
    \midrule
    \multicolumn{7}{@{}l}{\textit{Score-guided}} \\[2pt]
    ETD-SR-Euc (IS) & 0.00005 & 0.018 & 0.613 & 0.0017 & ---   & ---   \\
    ETD-BAL-Mom     & 0.0002  & 0.378 & 0.311 & 0.0008 & ---   & 0.303 \\
    ETD-SR-Euc      & 0.0013  & 110.8 & 0.088 & 0.067 & ---   & ---   \\
    \midrule
    \multicolumn{7}{@{}l}{\textit{Baselines}} \\[2pt]
    SVGD     & \multicolumn{6}{l}{step size $= 0.0272$} \\
    SGLD     & \multicolumn{6}{l}{$a = 6.58 \times 10^{-4}$,\; $b = 1.0$,\; $\gamma = 0.55$} \\
    AGF-SVGD & \multicolumn{6}{l}{lr $= 0.862$,\; annealing stages $= 85$,\; smoothing bw $= 0.136$,\; base scale $= 144.6$} \\
    \bottomrule
  \end{tabular}
  \label{tab-app:blr_german_tuned}
\end{table}

\begin{table}[h]
  \caption{BLR German Credit: Results for the top 3 score-free and top 3 score-guided ETD variants. NLL: negative log-likelihood (mean $\pm$ SE, 20 splits). Accuracy (mean $\pm$ SE, 20 splits).}
  \centering
  \small
  \begin{tabular}{@{}l cc@{}}
    \toprule
    \textbf{Method} & NLL $\downarrow$ & Accuracy $\uparrow$ \\
    \midrule
    \multicolumn{3}{@{}l}{\textit{Score-free}} \\[1pt]
    ETD-UB-Mom (IS)   & $0.4779 \pm 0.0111$ & $0.7770 \pm 0.0115$ \\
    ETD-UB-Euc        & $0.4783 \pm 0.0116$ & $0.7825 \pm 0.0094$ \\
    ETD-SR-Mom        & $0.4798 \pm 0.0122$ & $0.7795 \pm 0.0102$ \\
    \multicolumn{3}{@{}l}{\textit{Score-guided}} \\[1pt]
    ETD-SR-Euc (IS)   & $0.4768 \pm 0.0116$ & $0.7765 \pm 0.0108$ \\
    ETD-BAL-Mom       & $0.4777 \pm 0.0112$ & $0.7790 \pm 0.0107$ \\
    ETD-SR-Euc        & $0.4805 \pm 0.0108$ & $0.7715 \pm 0.0089$ \\
    \multicolumn{3}{@{}l}{\textit{Baselines}} \\[1pt]
    SVGD              & $0.4772 \pm 0.0111$ & $0.7780 \pm 0.0107$ \\
    SGLD              & $0.4765 \pm 0.0116$ & $0.7765 \pm 0.0107$ \\
    AGF-SVGD          & $0.5032 \pm 0.0135$ & $0.7710 \pm 0.0101$ \\
    \bottomrule
  \end{tabular}
  \label{tab-app:blr_german_results}
\end{table}

\subsubsection{Covertype}

\textbf{Tuning.} All methods are tuned on a fixed 70/10/20 train/validation/test split of Covertype with validation NLL as the tuning objective, averaged over 3 seeds per trial. Each method trains on the full 70\% training split, evaluates predictive metrics on the 20\% test split, and uses the validation split only for hyperparameter tuning. ETD variants and AGF-SVGD are tuned with 200 Optuna TPE trials each while SVGD and SGLD use 50 trials. Since full-data NUTS is computationally infeasible for Covertype, we construct a subsampled NUTS reference: for each of the 20 evaluation splits, we draw a 20K-point subsample from the training portion and run NUTS with 10 chains, 2000 warmup steps, and 2000 samples per chain, producing 20K posterior samples per split. The hyperparameters are provided in Table~\ref{tab-app:blr_covertype_tuned}.

\begin{table}[h]
  \caption{BLR Covertype: Tuned hyperparameters for selected score-free and score-guided ETD variants. All use $N=100$ particles and 5000 iterations with minibatch size 100 (except AGF-SVGD uses full-batch). For ETD, $M=500$ proposals.}
  \centering
  \small
  \begin{tabular}{@{}l cccccc@{}}
    \toprule
    \textbf{Method} & $\epsilon$ & $\beta$ & $\alpha$ & $\sigma$ & $\tau$ & $\mu$ \\
    \midrule
    \multicolumn{7}{@{}l}{\textit{Score-free}} \\[2pt]
    ETD-SR-Euc        & 2.865                & 0.013 & 0.550 & 0.0029 & --- & ---   \\
    ETD-BAL-Mom (IS)  & 1.569                & 35.03 & 0.391 & 0.0024 & --- & 0.017 \\
    ETD-BAL-Maha (IS) & 7.390                & 36.36 & 0.304 & 0.0025 & --- & ---   \\
    \midrule
    \multicolumn{7}{@{}l}{\textit{Score-guided}} \\[2pt]
    ETD-SR-Maha      & $3.6 \times 10^{-7}$ & 28.69 & 0.829 & 0.025 & --- & --- \\
    ETD-SR-Euc       & $3.6 \times 10^{-7}$ & 403.4 & 0.792 & 0.029 & --- & --- \\
    ETD-SR-Maha (IS) & $3.0 \times 10^{-7}$ & 18.13 & 0.704 & 0.016 & --- & --- \\
    \midrule
    \multicolumn{7}{@{}l}{\textit{Baselines}} \\[2pt]
    SVGD     & \multicolumn{6}{l}{step size $= 0.0484$} \\
    SGLD     & \multicolumn{6}{l}{$a = 4.59 \times 10^{-6}$,\; $b = 1.0$,\; $\gamma = 0.55$} \\
    AGF-SVGD & \multicolumn{6}{l}{lr $= 0.0984$,\; annealing stages $= 109$,\; smoothing bw $= 0.287$,\; base scale $= 136.7$} \\
    \bottomrule
  \end{tabular}
  \label{tab-app:blr_covertype_tuned}
\end{table}

\begin{table}[h]
  \caption{BLR Covertype: Results for the top 3 score-free and top 3 score-guided ETD variants. NLL: negative log-likelihood; Acc: accuracy; AUC: area under ROC curve; Brier: Brier score (all mean $\pm$ SE, 20 splits). Cov$_{90}$: marginal coverage of the particle 90\% credible intervals against the NUTS reference posterior, averaged over dimensions and splits (nominal $=0.9$).}
  \centering
  \scriptsize
  \setlength{\tabcolsep}{2.5pt}
  \begin{tabular}{@{}l ccccc@{}}
    \toprule
    \textbf{Method} & NLL $\downarrow$ & Acc $\uparrow$ & AUC $\uparrow$ & Brier $\downarrow$ & Cov$_{90}$ \\
    \midrule
    \multicolumn{6}{@{}l}{\textit{Score-free}} \\[1pt]
    ETD-SR-Euc        & $0.5183 \pm 0.0004$ & $0.7528 \pm 0.0005$ & $0.8246 \pm 0.0002$ & $0.1711 \pm 0.0001$ & $0.024$ \\
    ETD-BAL-Mom (IS)  & $0.5184 \pm 0.0002$ & $0.7534 \pm 0.0005$ & $0.8245 \pm 0.0002$ & $0.1713 \pm 0.0001$ & $0.000$ \\
    ETD-BAL-Maha (IS) & $0.5187 \pm 0.0002$ & $0.7535 \pm 0.0005$ & $0.8244 \pm 0.0002$ & $0.1714 \pm 0.0001$ & $0.000$ \\
    \multicolumn{6}{@{}l}{\textit{Score-guided}} \\[1pt]
    ETD-SR-Maha       & $0.5147 \pm 0.0003$ & $0.7549 \pm 0.0007$ & $0.8263 \pm 0.0002$ & $0.1703 \pm 0.0001$ & $0.986$ \\
    ETD-SR-Euc        & $0.5150 \pm 0.0002$ & $0.7557 \pm 0.0004$ & $0.8264 \pm 0.0001$ & $0.1701 \pm 0.0001$ & $0.986$ \\
    ETD-SR-Maha (IS)  & $0.5150 \pm 0.0002$ & $0.7545 \pm 0.0005$ & $0.8261 \pm 0.0002$ & $0.1704 \pm 0.0001$ & $0.979$ \\
    \multicolumn{6}{@{}l}{\textit{Baselines}} \\[1pt]
    SVGD              & $0.5144 \pm 0.0002$ & $0.7554 \pm 0.0003$ & $0.8266 \pm 0.0002$ & $0.1703 \pm 0.0001$ & $0.323$ \\
    SGLD              & $0.5145 \pm 0.0002$ & $0.7556 \pm 0.0002$ & $0.8266 \pm 0.0001$ & $0.1703 \pm 0.0001$ & $0.277$ \\
    AGF-SVGD          & $0.5149 \pm 0.0003$ & $0.7559 \pm 0.0002$ & $0.8266 \pm 0.0002$ & $0.1703 \pm 0.0001$ & $0.201$ \\
    \bottomrule
  \end{tabular}
  \label{tab-app:blr_covertype_results}
\end{table}

\subsection{Bayesian Neural Networks}

\textbf{Description.} Following the standard BNN regression setup used by \citet{Liu16}, we use a one-hidden-layer ReLU network with 50 hidden units and Gamma hyperpriors on the noise and weight precisions. Our implementation uses $\mathrm{Gamma}(1, 0.01)$ for both precisions and full-batch likelihood evaluation. We evaluate on 9 of the 10 standard UCI datasets, excluding Year.

\textbf{Tuning.} All methods are tuned on the Boston Housing dataset on a single 90/10 train/test split with NLL as the tuning objective. ETD variants and AGF-SVGD are tuned with 200 Optuna TPE trials each while SVGD and SGLD use 50 trials. Tuned hyperparameters are applied to all 9 datasets without per-dataset retuning. The hyperparameters are provided in Table \ref{tab-app:bnn_tuned}.

\begin{table}[h]
  \caption{BNN UCI: Tuned hyperparameters for ETD variants shown in Tables~\ref{tab-app:bnn_results_1}--\ref{tab-app:bnn_results_2}. All use $N=100$ particles and 2000 iterations. For ETD, $M=500$ proposals.}
  \centering
  \scriptsize
  \begin{tabular}{@{}l cccccc@{}}
    \toprule
    \textbf{Method} & $\epsilon$ & $\beta$ & $\alpha$ & $\sigma$ & $\tau$ & $\mu$ \\
    \midrule
    \multicolumn{7}{@{}l}{\textit{Score-free}} \\[2pt]
    ETD-BAL-Euc       & 0.0045 & 0.165 & 0.318 & 0.012 & ---   & ---   \\
    ETD-BAL-Maha (IS) & 1.048  & 5.591 & 0.286 & 0.009 & ---   & ---   \\
    ETD-BAL-Mom       & 0.0706 & 0.102 & 0.248 & 0.005 & ---   & 0.385 \\
    ETD-SR-Euc        & 0.343  & 0.126 & 0.313 & 0.009 & ---   & ---   \\
    ETD-SR-Mom        & 0.780  & 0.208 & 0.236 & 0.013 & ---   & 0.098 \\
    ETD-UB-Euc        & 3.217  & 0.149 & 0.419 & 0.011 & 22.70 & ---   \\
    ETD-UB-Mom        & 2.784  & 0.133 & 0.345 & 0.009 & 4.025 & 0.167 \\
    \midrule
    \multicolumn{7}{@{}l}{\textit{Score-guided}} \\[2pt]
    ETD-SR-Euc        & 0.0001 & 124.0 & 0.444 & 0.023 & ---   & ---   \\
    ETD-SR-Euc (IS)   & 0.0001 & 0.151 & 0.885 & 0.010 & ---   & ---   \\
    ETD-SR-Maha       & 0.0001 & 35.14 & 0.488 & 0.016 & ---   & ---   \\
    ETD-SR-Maha (IS)  & 0.0001 & 0.018 & 0.737 & 0.011 & ---   & ---   \\
    ETD-SR-Mom        & 0.0001 & 71.37 & 0.834 & 0.009 & ---   & 0.462 \\
    ETD-SR-Mom (IS)   & 0.0001 & 63.48 & 0.665 & 0.010 & ---   & 0.407 \\
    \midrule
    \multicolumn{7}{@{}l}{\textit{Baselines}} \\[2pt]
    SVGD     & \multicolumn{6}{l}{step size $= 0.0349$} \\
    SGLD     & \multicolumn{6}{l}{$a = 0.00159$,\; $b = 1.0$,\; $\gamma = 0.55$} \\
    AGF-SVGD & \multicolumn{6}{l}{lr $= 0.00304$,\; annealing stages $= 65$,\; smoothing bw $= 0.086$,\; base scale $= 141.3$} \\
    \bottomrule
  \end{tabular}
  \label{tab-app:bnn_tuned}
\end{table}

\begin{table}[h]
  \caption{BNN UCI (Part 1): Results for the top 3 score-free and top 3 score-guided ETD variants per dataset. NLL: test negative log-likelihood; RMSE: root mean squared error (both mean $\pm$ SE, 20 train/test splits except Protein, which uses 5 splits.).}
  \centering
  \scriptsize
  \setlength{\tabcolsep}{2.5pt}

  \begin{minipage}[t]{0.48\textwidth}\centering
  \textbf{Boston}\\[3pt]
  \begin{tabular}{@{}l cc@{}}
    \toprule
    \textbf{Method} & NLL & RMSE \\
    \midrule
    \multicolumn{3}{@{}l}{\textit{Score-free}} \\[1pt]
    ETD-UB-Mom       & $2.621 \pm 0.071$ & $3.186 \pm 0.148$ \\
    ETD-BAL-Mom      & $2.625 \pm 0.096$ & $3.115 \pm 0.189$ \\
    ETD-BAL-Euc      & $2.625 \pm 0.073$ & $3.186 \pm 0.156$ \\
    \multicolumn{3}{@{}l}{\textit{Score-guided}} \\[1pt]
    ETD-SR-Mom (IS)  & $2.495 \pm 0.040$ & $2.895 \pm 0.171$ \\
    ETD-SR-Mom       & $2.498 \pm 0.055$ & $2.913 \pm 0.176$ \\
    ETD-SR-Maha (IS) & $2.530 \pm 0.043$ & $3.006 \pm 0.177$ \\
    \multicolumn{3}{@{}l}{\textit{Baselines}} \\[1pt]
    SVGD     & $2.492 \pm 0.085$ & $2.776 \pm 0.160$ \\
    SGLD     & $2.513 \pm 0.056$ & $2.938 \pm 0.174$ \\
    AGF-SVGD & $3.419 \pm 0.014$ & $5.761 \pm 0.268$ \\
    \bottomrule
  \end{tabular}
  \end{minipage}%
  \hfill
  \begin{minipage}[t]{0.48\textwidth}\centering
  \textbf{Concrete}\\[3pt]
  \begin{tabular}{@{}l cc@{}}
    \toprule
    \textbf{Method} & NLL & RMSE \\
    \midrule
    \multicolumn{3}{@{}l}{\textit{Score-free}} \\[1pt]
    ETD-BAL-Mom     & $3.095 \pm 0.022$ & $5.295 \pm 0.101$ \\
    ETD-SR-Euc      & $3.149 \pm 0.019$ & $5.605 \pm 0.091$ \\
    ETD-UB-Euc      & $3.159 \pm 0.032$ & $5.647 \pm 0.152$ \\
    \multicolumn{3}{@{}l}{\textit{Score-guided}} \\[1pt]
    ETD-SR-Mom      & $3.075 \pm 0.012$ & $5.077 \pm 0.111$ \\
    ETD-SR-Mom (IS) & $3.108 \pm 0.011$ & $5.139 \pm 0.111$ \\
    ETD-SR-Euc (IS) & $3.151 \pm 0.010$ & $5.398 \pm 0.105$ \\
    \multicolumn{3}{@{}l}{\textit{Baselines}} \\[1pt]
    SVGD     & $3.013 \pm 0.023$ & $4.945 \pm 0.111$ \\
    SGLD     & $3.047 \pm 0.014$ & $5.115 \pm 0.106$ \\
    AGF-SVGD & $4.072 \pm 0.009$ & $12.647 \pm 0.251$ \\
    \bottomrule
  \end{tabular}
  \end{minipage}

  \vspace{10pt}

  \begin{minipage}[t]{0.48\textwidth}\centering
  \textbf{Energy}\\[3pt]
  \begin{tabular}{@{}l cc@{}}
    \toprule
    \textbf{Method} & NLL & RMSE \\
    \midrule
    \multicolumn{3}{@{}l}{\textit{Score-free}} \\[1pt]
    ETD-BAL-Mom     & $1.217 \pm 0.018$ & $0.772 \pm 0.019$ \\
    ETD-SR-Euc      & $1.418 \pm 0.021$ & $0.987 \pm 0.022$ \\
    ETD-UB-Mom      & $1.468 \pm 0.029$ & $1.031 \pm 0.030$ \\
    \multicolumn{3}{@{}l}{\textit{Score-guided}} \\[1pt]
    ETD-SR-Mom      & $2.237 \pm 0.008$ & $1.744 \pm 0.029$ \\
    ETD-SR-Mom (IS) & $2.254 \pm 0.008$ & $1.681 \pm 0.031$ \\
    ETD-SR-Euc (IS) & $2.272 \pm 0.008$ & $1.714 \pm 0.038$ \\
    \multicolumn{3}{@{}l}{\textit{Baselines}} \\[1pt]
    SVGD     & $1.756 \pm 0.004$ & $0.761 \pm 0.021$ \\
    SGLD     & $1.991 \pm 0.006$ & $1.226 \pm 0.031$ \\
    AGF-SVGD & $3.462 \pm 0.008$ & $5.389 \pm 0.189$ \\
    \bottomrule
  \end{tabular}
  \end{minipage}%
  \hfill
  \begin{minipage}[t]{0.48\textwidth}\centering
  \textbf{Kin8nm}\\[3pt]
  \begin{tabular}{@{}l cc@{}}
    \toprule
    \textbf{Method} & NLL & RMSE \\
    \midrule
    \multicolumn{3}{@{}l}{\textit{Score-free}} \\[1pt]
    ETD-BAL-Mom     & $-0.936 \pm 0.009$ & $0.095 \pm 0.001$ \\
    ETD-SR-Euc      & $-0.852 \pm 0.008$ & $0.103 \pm 0.001$ \\
    ETD-UB-Mom      & $-0.851 \pm 0.007$ & $0.103 \pm 0.001$ \\
    \multicolumn{3}{@{}l}{\textit{Score-guided}} \\[1pt]
    ETD-SR-Euc (IS) & $-0.388 \pm 0.014$ & $0.136 \pm 0.001$ \\
    ETD-SR-Mom      & $-0.311 \pm 0.009$ & $0.149 \pm 0.002$ \\
    ETD-SR-Mom (IS) & $-0.307 \pm 0.016$ & $0.145 \pm 0.002$ \\
    \multicolumn{3}{@{}l}{\textit{Baselines}} \\[1pt]
    SVGD     & $-1.087 \pm 0.005$ & $0.082 \pm 0.001$ \\
    SGLD     & $-0.220 \pm 0.013$ & $122.0 \pm 18.3$ \\
    AGF-SVGD & $-0.034 \pm 0.007$ & $0.221 \pm 0.003$ \\
    \bottomrule
  \end{tabular}
  \end{minipage}
  \label{tab-app:bnn_results_1}
\end{table}

\begin{table}[h]
  \caption{BNN UCI (Part 2): Continued from Table~\ref{tab-app:bnn_results_1}.}
  \centering
  \scriptsize
  \setlength{\tabcolsep}{2.5pt}

  \begin{minipage}[t]{0.48\textwidth}\centering
  \textbf{Naval}\\[3pt]
  \begin{tabular}{@{}l cc@{}}
    \toprule
    \textbf{Method} & NLL & RMSE \\
    \midrule
    \multicolumn{3}{@{}l}{\textit{Score-free}} \\[1pt]
    ETD-BAL-Mom       & $-5.110 \pm 0.017$ & $0.001 \pm 0.000$ \\
    ETD-SR-Euc        & $-4.965 \pm 0.018$ & $0.002 \pm 0.000$ \\
    ETD-BAL-Maha (IS) & $-4.920 \pm 0.015$ & $0.002 \pm 0.000$ \\
    \multicolumn{3}{@{}l}{\textit{Score-guided}} \\[1pt]
    ETD-SR-Euc (IS)   & $-2.899 \pm 0.003$ & $0.014 \pm 0.001$ \\
    ETD-SR-Maha (IS)  & $-2.875 \pm 0.007$ & $0.017 \pm 0.003$ \\
    ETD-SR-Mom (IS)   & $-2.782 \pm 0.009$ & $0.034 \pm 0.006$ \\
    \multicolumn{3}{@{}l}{\textit{Baselines}} \\[1pt]
    SVGD     & $-3.780 \pm 0.006$ & $0.006 \pm 0.000$ \\
    SGLD     & $3.007 \pm 0.346$ & $3.12{\times}10^4 \pm 5.34{\times}10^3$ \\
    AGF-SVGD & $-2.796 \pm 0.002$ & $0.015 \pm 0.000$ \\
    \bottomrule
  \end{tabular}
  \end{minipage}%
  \hfill
  \begin{minipage}[t]{0.48\textwidth}\centering
  \textbf{Power}\\[3pt]
  \begin{tabular}{@{}l cc@{}}
    \toprule
    \textbf{Method} & NLL & RMSE \\
    \midrule
    \multicolumn{3}{@{}l}{\textit{Score-free}} \\[1pt]
    ETD-BAL-Mom      & $2.805 \pm 0.012$ & $3.987 \pm 0.048$ \\
    ETD-UB-Mom       & $2.825 \pm 0.012$ & $4.068 \pm 0.048$ \\
    ETD-SR-Euc       & $2.825 \pm 0.012$ & $4.072 \pm 0.050$ \\
    \multicolumn{3}{@{}l}{\textit{Score-guided}} \\[1pt]
    ETD-SR-Maha (IS) & $3.931 \pm 0.040$ & $9.552 \pm 0.406$ \\
    ETD-SR-Euc (IS)  & $4.112 \pm 0.126$ & $15.15 \pm 1.87$ \\
    ETD-SR-Maha      & $4.930 \pm 0.345$ & $220.2 \pm 131.3$ \\
    \multicolumn{3}{@{}l}{\textit{Baselines}} \\[1pt]
    SVGD     & $2.835 \pm 0.008$ & $4.025 \pm 0.046$ \\
    SGLD     & $9.290 \pm 0.315$ & $1.29{\times}10^7 \pm 1.06{\times}10^6$ \\
    AGF-SVGD & $3.959 \pm 0.016$ & $9.199 \pm 0.535$ \\
    \bottomrule
  \end{tabular}
  \end{minipage}

  \vspace{10pt}

  \begin{minipage}[t]{0.48\textwidth}\centering
  \textbf{Protein}\\[3pt]
  \begin{tabular}{@{}l cc@{}}
    \toprule
    \textbf{Method} & NLL & RMSE \\
    \midrule
    \multicolumn{3}{@{}l}{\textit{Score-free}} \\[1pt]
    ETD-BAL-Mom       & $2.936 \pm 0.003$ & $4.557 \pm 0.014$ \\
    ETD-SR-Euc        & $2.951 \pm 0.001$ & $4.630 \pm 0.007$ \\
    ETD-UB-Mom        & $2.955 \pm 0.002$ & $4.648 \pm 0.011$ \\
    \multicolumn{3}{@{}l}{\textit{Score-guided}} \\[1pt]
    ETD-SR-Euc        & $3.238 \pm 0.051$ & $5.882 \pm 0.173$ \\
    ETD-SR-Maha (IS)  & $4.299 \pm 0.729$ & $80.4 \pm 65.3$ \\
    ETD-SR-Maha       & $5.157 \pm 0.851$ & $321.2 \pm 269.3$ \\
    \multicolumn{3}{@{}l}{\textit{Baselines}} \\[1pt]
    SVGD     & $2.937 \pm 0.002$ & $4.565 \pm 0.008$ \\
    SGLD     & $16.986 \pm 0.764$ & $1.52{\times}10^{10} \pm 6.87{\times}10^9$ \\
    AGF-SVGD & $3.193 \pm 0.006$ & $5.832 \pm 0.047$ \\
    \bottomrule
  \end{tabular}
  \end{minipage}%
  \hfill
  \begin{minipage}[t]{0.48\textwidth}\centering
  \textbf{Wine}\\[3pt]
  \begin{tabular}{@{}l cc@{}}
    \toprule
    \textbf{Method} & NLL & RMSE \\
    \midrule
    \multicolumn{3}{@{}l}{\textit{Score-free}} \\[1pt]
    ETD-BAL-Mom     & $0.958 \pm 0.011$ & $0.630 \pm 0.007$ \\
    ETD-UB-Euc      & $0.963 \pm 0.011$ & $0.633 \pm 0.007$ \\
    ETD-SR-Mom      & $0.963 \pm 0.011$ & $0.634 \pm 0.007$ \\
    \multicolumn{3}{@{}l}{\textit{Score-guided}} \\[1pt]
    ETD-SR-Mom (IS) & $0.950 \pm 0.010$ & $0.628 \pm 0.007$ \\
    ETD-SR-Mom      & $0.953 \pm 0.010$ & $0.629 \pm 0.007$ \\
    ETD-SR-Euc (IS) & $0.954 \pm 0.010$ & $0.629 \pm 0.007$ \\
    \multicolumn{3}{@{}l}{\textit{Baselines}} \\[1pt]
    SVGD     & $0.954 \pm 0.014$ & $0.614 \pm 0.007$ \\
    SGLD     & $0.953 \pm 0.010$ & $0.629 \pm 0.007$ \\
    AGF-SVGD & $1.093 \pm 0.008$ & $0.684 \pm 0.010$ \\
    \bottomrule
  \end{tabular}
  \end{minipage}

  \vspace{10pt}

  \begin{minipage}[t]{0.48\textwidth}\centering
  \textbf{Yacht}\\[3pt]
  \begin{tabular}{@{}l cc@{}}
    \toprule
    \textbf{Method} & NLL & RMSE \\
    \midrule
    \multicolumn{3}{@{}l}{\textit{Score-free}} \\[1pt]
    ETD-BAL-Mom     & $1.554 \pm 0.055$ & $1.071 \pm 0.081$ \\
    ETD-SR-Euc      & $1.560 \pm 0.072$ & $1.160 \pm 0.074$ \\
    ETD-UB-Mom      & $1.588 \pm 0.070$ & $1.185 \pm 0.078$ \\
    \multicolumn{3}{@{}l}{\textit{Score-guided}} \\[1pt]
    ETD-SR-Mom      & $1.975 \pm 0.010$ & $1.084 \pm 0.078$ \\
    ETD-SR-Mom (IS) & $2.095 \pm 0.011$ & $1.088 \pm 0.073$ \\
    ETD-SR-Euc (IS) & $2.119 \pm 0.008$ & $1.207 \pm 0.071$ \\
    \multicolumn{3}{@{}l}{\textit{Baselines}} \\[1pt]
    SVGD     & $2.144 \pm 0.011$ & $1.268 \pm 0.082$ \\
    SGLD     & $1.832 \pm 0.009$ & $0.929 \pm 0.066$ \\
    AGF-SVGD & $3.973 \pm 0.015$ & $11.303 \pm 0.451$ \\
    \bottomrule
  \end{tabular}
  \end{minipage}
  \label{tab-app:bnn_results_2}
\end{table}

\subsection{Molecular Boltzmann Distribution}

\subsubsection{Double-Well 4 (DW-4)}

\textbf{Description.} The DW-4 benchmark consists of 4 particles in 2D ($d=8$) interacting via a pairwise double-well potential: $V_{\text{DW}}(d_{ij}) = a(d_{ij} - d_0)^4 + b(d_{ij} - d_0)^2$, where $d_{ij}$ is the Euclidean distance between particles $i$ and $j$, with parameters $a=0.9$, $b=-4$, $d_0=4$. The total energy is $U(x) = \sum_{i<j}  V_{\text{DW}}(d_{ij})$ and the target distribution is $\pi(x) \propto \exp(-U(x)/k_{B} T)$ at unit temperature $k_{B} T=1$. Following the standard translation-invariant formulation, we subtract the center of mass of each configuration before evaluating pairwise distances and the energy.

\textbf{Tuning.} All methods are tuned via Optuna TPE  by minimizing the total variation distance between the pairwise distance histograms of sampled and reference configurations, computed against the validation split. Each trial evaluates the objective averaged over 3 seeds to reduce tuning variance. ETD methods and AGF-SVGD are tuned with 200 trials each while SVGD and SGLD use 50 trials due to their smaller search spaces. Both tuned and fixed hyperparameters  for the top 3 score-free and score guided ETD variants are provided in Table \ref{tab-app:dw4_tuned}.

\subsubsection{Lennard-Jones-13 (LJ-13)}

\textbf{Description.} The LJ-13 benchmark consists of 13 atoms in 3D  ($d=39$) interacting via the Lennard-Jones potential with harmonic confinement: 

$$U(x) = \lambda \sum_{i<j} \varepsilon\left[\left(\frac{r_m}{r_{ij}}\right)^{12} - 2\left(\frac{r_m}{r_{ij}}\right)^{6}\right] + \frac{1}{2}\kappa\sum_i \|x_i -  x_{\text{cm}}\|^2$$ 

  with $\varepsilon=1$, $r_m=1$, $\kappa=1$, $k_B T=1$, and $\lambda=2$ following the DEM \citep{akhound2024iterated} convention. The target distribution is $\pi(x) \propto \exp(-U(x)/(k_B T)) = \exp(-U(x))$. The system has a complex energy landscape and many local minima, making it a substantially harder multimodal target than DW-4. As with DW-4, we subtract the center of mass before evaluating distances and the energy.

\textbf{Reference data.} MCMC reference samples were obtained from the DEM repository: 100K training, 10K validation, and 10K test samples.

\textbf{Tuning.} We follow the exact same tuning strategy as the DW-4 experiment except LJ-13 uses 5000 iterations. The hyperparameters are provided in Table \ref{tab-app:lj13_tuned}.

\begin{table}[h]
  \caption{DW-4: Tuned hyperparameters for the top 3 score-free and top 3 score-guided ETD variants (filtered to ETD variants with 20/20 energy-stable seeds). All use $N=100$ particles and 2000 iterations. For ETD, $M=500$ proposals.}
  \centering
  \small
  \begin{tabular}{@{}l cccccc@{}}
    \toprule
    \textbf{Method} & $\epsilon$ & $\beta$ & $\alpha$ & $\sigma$ & $\tau$ & $\mu$ \\
    \midrule
    \multicolumn{7}{@{}l}{\textit{Score-free}} \\[2pt]
    ETD-UB-Euc   & 0.0071 & 0.524 & 0.728 & 0.171 & 0.019 & ---   \\
    ETD-UB-Mom   & 0.708  & 1.584 & 0.900 & 0.229 & 0.462 & 0.011 \\
    ETD-BAL-Maha & 0.421  & 0.267 & 0.784 & 0.139 & ---   & ---   \\
    \midrule
    \multicolumn{7}{@{}l}{\textit{Score-guided}} \\[2pt]
    ETD-SR-Euc  & 0.0068 & 0.013 & 0.942 & 0.093 & ---   & ---   \\
    ETD-UB-Maha & 0.0061 & 0.205 & 0.644 & 0.124 & 0.014 & ---   \\
    ETD-UB-Euc  & 0.0211 & 0.074 & 0.912 & 0.172 & 0.056 & ---   \\
    \midrule
    \multicolumn{7}{@{}l}{\textit{Baselines}} \\[2pt]
    SVGD     & \multicolumn{6}{l}{step size $= 0.0349$} \\
    SGLD     & \multicolumn{6}{l}{$a = 0.0185$,\; $b = 1.0$,\; $\gamma = 0.55$} \\
    AGF-SVGD & \multicolumn{6}{l}{lr $= 0.440$,\; annealing stages $= 431$,\; smoothing bw $= 0.074$,\; base scale $= 1.83$} \\
    \bottomrule
  \end{tabular}
  \label{tab-app:dw4_tuned}
\end{table}

\begin{table}[h]
  \caption{DW-4: Results for the top 3 score-free and top 3 score-guided ETD variants (filtered to ETD variants with 20/20 energy-stable seeds). TV: total variation of interatomic distance distributions; $E_\text{dist}$: energy distance (both mean $\pm$ SE, 20 seeds). $\bar{E}$: mean energy of sampled configurations (reference $\bar{E}_\text{ref} = -22.45$).\textsuperscript{$\ast$}AGF-SVGD has high-energy failures in 4/20 seeds.}
  \centering
  \small
  \begin{tabular}{@{}l ccc@{}}
    \toprule
    \textbf{Method} & TV $\downarrow$ & $E_\text{dist}$ $\downarrow$ & $\bar{E}$ \\
    \midrule
    \multicolumn{4}{@{}l}{\textit{Score-free}} \\[1pt]
    ETD-UB-Euc   & $0.193 \pm 0.006$ & $0.587 \pm 0.064$ & $-22.24 \pm 0.06$ \\
    ETD-UB-Mom   & $0.210 \pm 0.006$ & $0.923 \pm 0.098$ & $-23.00 \pm 0.11$ \\
    ETD-BAL-Maha & $0.218 \pm 0.004$ & $1.340 \pm 0.049$ & $-22.64 \pm 0.08$ \\
    \multicolumn{4}{@{}l}{\textit{Score-guided}} \\[1pt]
    ETD-SR-Euc   & $0.173 \pm 0.004$ & $0.256 \pm 0.039$ & $-21.90 \pm 0.09$ \\
    ETD-UB-Maha  & $0.181 \pm 0.004$ & $0.361 \pm 0.040$ & $-21.97 \pm 0.06$ \\
    ETD-UB-Euc   & $0.183 \pm 0.004$ & $0.346 \pm 0.062$ & $-21.58 \pm 0.06$ \\
    \multicolumn{4}{@{}l}{\textit{Baselines}} \\[1pt]
    SVGD                            & $0.216 \pm 0.003$ & $0.233 \pm 0.019$ & $-22.29 \pm 0.18$ \\
    SGLD                            & $0.162 \pm 0.003$ & $0.184 \pm 0.017$ & $-22.17 \pm 0.05$ \\
    AGF-SVGD\textsuperscript{$\ast$} & $0.531 \pm 0.005$ & $0.616 \pm 0.035$ & $-2.97 \pm 0.76$ \\
    \bottomrule
  \end{tabular}
  \label{tab-app:dw4_results}
\end{table}

\begin{table}[h]
  \caption{LJ-13: Tuned hyperparameters for the top 3 score-free and top 3 score-guided ETD variants (filtered to ETD variants with 20/20 energy-stable seeds). All use $N=100$ particles and 5000 iterations. For ETD, $M=500$ proposals.}
  \centering
  \small
  \begin{tabular}{@{}l cccccc@{}}
    \toprule
    \textbf{Method} & $\epsilon$ & $\beta$ & $\alpha$ & $\sigma$ & $\tau$ & $\mu$ \\
    \midrule
    \multicolumn{7}{@{}l}{\textit{Score-free}} \\[2pt]
    ETD-BAL-Euc & 0.141  & 0.057 & 0.929 & 0.019 & --- & ---   \\
    ETD-SR-Euc  & 0.039  & 0.101 & 0.947 & 0.017 & --- & ---   \\
    ETD-SR-Maha & 0.038  & 0.100 & 0.856 & 0.019 & --- & ---   \\
    \midrule
    \multicolumn{7}{@{}l}{\textit{Score-guided}} \\[2pt]
    ETD-SR-Maha & 0.0004 & 0.148 & 0.735 & 0.021 & --- & ---   \\
    ETD-SR-Euc  & 0.0004 & 0.161 & 0.908 & 0.024 & --- & ---   \\
    ETD-SR-Mom  & 0.0007 & 0.001 & 0.827 & 0.027 & --- & 0.252 \\
    \midrule
    \multicolumn{7}{@{}l}{\textit{Baselines}} \\[2pt]
    SVGD     & \multicolumn{6}{l}{step size $= 6.94 \times 10^{-5}$} \\
    SGLD     & \multicolumn{6}{l}{$a = 1.00 \times 10^{-8}$,\; $b = 1.0$,\; $\gamma = 0.55$} \\
    AGF-SVGD & \multicolumn{6}{l}{lr $= 5.46 \times 10^{-5}$,\; annealing stages $= 238$,\; smoothing bw $= 0.112$,\; base scale $= 101.3$} \\
    \bottomrule
  \end{tabular}
  \label{tab-app:lj13_tuned}
\end{table}

\begin{table}[h]
  \caption{LJ-13: Results for the top 3 score-free and top 3 score-guided ETD variants (filtered to ETD variants with 20/20 energy-stable seeds). TV: total variation of interatomic distance distributions; $E_\text{dist}$: energy distance (both mean $\pm$ SE, 20 seeds). $\bar{E}$: mean energy of sampled configurations (reference $\bar{E}_\text{ref} = -43.13$). Entries marked div. indicate numerically divergent/high-energy samples.\textsuperscript{$\ast$}All 20 seeds produced divergent energies.}
  \centering
  \small
  \begin{tabular}{@{}l ccc@{}}
    \toprule
    \textbf{Method} & TV $\downarrow$ & $E_\text{dist}$ $\downarrow$ & $\bar{E}$ \\
    \midrule
    \multicolumn{4}{@{}l}{\textit{Score-free}} \\[1pt]
    ETD-BAL-Euc & $0.113 \pm 0.004$ & $0.530 \pm 0.062$ & $-40.01 \pm 0.70$ \\
    ETD-SR-Euc  & $0.122 \pm 0.006$ & $0.772 \pm 0.071$ & $-41.16 \pm 1.10$ \\
    ETD-SR-Maha & $0.128 \pm 0.007$ & $0.856 \pm 0.106$ & $-42.24 \pm 1.08$ \\
    \multicolumn{4}{@{}l}{\textit{Score-guided}} \\[1pt]
    ETD-SR-Maha & $0.053 \pm 0.001$ & $0.058 \pm 0.012$ & $-44.43 \pm 0.26$ \\
    ETD-SR-Euc  & $0.054 \pm 0.001$ & $0.036 \pm 0.005$ & $-41.35 \pm 0.18$ \\
    ETD-SR-Mom  & $0.054 \pm 0.001$ & $0.032 \pm 0.004$ & $-43.30 \pm 0.21$ \\
    \multicolumn{4}{@{}l}{\textit{Baselines}} \\[1pt]
    SVGD\textsuperscript{$\ast$}     & $0.394 \pm 0.002$ & $8.220 \pm 0.093$ & div. \\
    SGLD\textsuperscript{$\ast$}     & $0.422 \pm 0.002$ & div. & div. \\
    AGF-SVGD\textsuperscript{$\ast$} & $0.394 \pm 0.002$ & $8.232 \pm 0.093$ & div. \\
    \bottomrule
  \end{tabular}
  \label{tab-app:lj13_results}
\end{table}

\clearpage

\end{document}